%% file: main.tex
\documentclass{article}
\usepackage{algorithm}
\usepackage{algpseudocode}
\usepackage{xcolor}
\usepackage{amsmath}
\usepackage{dsfont}
\usepackage{subcaption}
\usepackage{amssymb}
\usepackage{graphicx}
\usepackage{multirow}
\usepackage{booktabs}
\usepackage{wrapfig}
\usepackage[title, titletoc]{appendix}
\usepackage{amsthm}
\usepackage{svg}
\usepackage[final]{corl_2026} 
\makeatletter
\renewcommand{\@conferencelocation}{Austin, Texas, USA}
\makeatother
\hypersetup{
  pdftitle={Gradient-Free Neural Hamilton-Jacobi Reachability for Scalable Safety-Critical Control},
  pdfauthor={Zeyuan Feng, Ali Fuat Sahin, Santiago Thorup, Somil Bansal}
}
\theoremstyle{plain} 
\newtheorem{theorem}{Theorem}
\theoremstyle{remark}
\newtheorem{remark}{Remark}
\algtext*{EndWhile}
\algtext*{EndIf}
\algtext*{EndFor}

\title{Gradient-Free Neural Hamilton-Jacobi Reachability for Scalable Safety-Critical Control}

\author{
  Zeyuan Feng\thanks{Equal contribution.}\\
  Department of Aeronautics and Astronautics\\
  Stanford University, United States\\
  \texttt{zeyuanf@stanford.edu} \\
  \And
  Ali Fuat Sahin\footnotemark[1] \\
  Department of Mechanical Engineering \\
  École Polytechnique Fédérale de Lausanne, Switzerland \\
  \texttt{alifuatpc@gmail.com} \\
  \And
  Santiago Thorup \\
  Department of Aeronautics and Astronautics\\
  Stanford University, United States\\
  \texttt{sthorup4@stanford.edu} \\
  \And
  Somil Bansal \\
  Department of Aeronautics and Astronautics\\
  Stanford University, United States\\
  \texttt{somil@stanford.edu} \\
}
\algrenewcommand\algorithmicindent{0.5em}

\begin{document}
\vspace{-2em}
\maketitle
\vspace{-2em}
\begin{abstract}
    Hamilton-Jacobi (HJ) reachability provides a principled framework for synthesizing safety certificates and robust controllers for safety-critical robotic systems. However, applying reachability analysis to high-dimensional nonlinear systems remains challenging: classical grid-based solvers suffer from the curse of dimensionality, continuous-time neural solvers require accurate spatial value gradients, and reinforcement-learning-based approaches often suffer from weak boundary anchoring and non-stationary adversarial policy optimization.
    We propose a discrete-time neural reachability framework for control-disturbance-affine systems that learns backward reachable tubes (BRTs) and backward reach-avoid tubes (BRATs) through Bellman-Isaacs value propagation.
    Our key idea is to combine equation-driven self-supervision with structured policy learning: rather than computing explicit PDE-gradients, we exploit the bang-bang structure of optimal safety interventions to construct approximate teacher actions from gradient-free value probes, converting adversarial actor learning into supervised policy learning. 
    To stabilize long-horizon value propagation, we leverage the learned actor to train the value function backward from the terminal boundary using a windowed temporal curriculum, where each window is used as the boundary condition for the next window.
    Across benchmark problems up to 80 dimensions, our method learns accurate reachability value functions while improving stability over existing learning-based solvers. We further demonstrate observation-space scalability on F1Tenth racing with over 16,000-dimensional egocentric inputs. The learned safety filter generalizes zero-shot to unseen tracks and transfers to a physical RC car, achieving real-time robust collision avoidance.\footnote{The code is available at \url{https://github.com/TheZeyuanFeng/grad_free_reach/tree/main} and the project webpage can be found at: \url{https://alifuatsahin.github.io/grad_free_reach/}}
\end{abstract}

\keywords{Robot Safety, Hamilton-Jacobi Reachability, Safety-Critical Control}

\input{intro}
\input{formulation}

\input{method}


\input{results}

\input{conclusion}

\input{limitations}

\clearpage
\acknowledgments{This research is supported in part by the DARPA Assured Neuro Symbolic Learning and Reasoning (ANSR) program, the NSF CAREER program (2240163), and the TRI University Research Program (URP) 3.0.}
\bibliography{references, bansal_papers}  

\input{appendix}

\end{document}

%% file: intro.tex
\vspace{-0.5em}
\section{Introduction}
\vspace{-0.5em}
From autonomous vehicles to aerial robots, robotic systems are increasingly deployed in safety-critical settings. In these domains, failures can arise not only from nominal planning errors, but also from disturbances, model mismatch, and adversarial interactions. This requires control policies and safety certificates that are robust to worst-case disturbances. Hamilton-Jacobi (HJ) reachability provides a principled framework for this purpose: by solving a dynamic game between the controller and an adversarial disturbance, reachability analysis computes the set of states from which safety can be guaranteed \cite{mitchell2005time}, or a target can be reached while avoiding failure \cite{Margellos11}.

Despite its strong theoretical foundations, HJ reachability remains difficult to apply to high-dimensional nonlinear robotic systems. Classical grid-based methods \cite{mitchell2007toolbox, bui2022optimizeddp} solve the Hamilton-Jacobi-Isaacs (HJI) partial differential equation with strong numerical guarantees, but their computational cost grows exponentially with the state dimension \cite{bansal2017hamilton}. 
To overcome this curse of dimensionality, recent works have explored learning-based approximations of HJ reachability. 
One prominent direction uses Physics-Informed Neural Networks (PINNs) to solve the continuous-time HJI PDE by minimizing its residual \cite{bansal2021deepreach, feng2025bridging, 11127972, chilakamarri2024reachability, sharpless2024linear}. The key appeal of this approach is that learning is self-supervised: the PDE itself provides a dense training signal, allowing the value function to be optimized without requiring ground-truth reachable sets. 
However, evaluating the HJI residual requires accurate spatial value gradients, $\nabla_x V$, throughout the state space. 
In long-horizon, high-dimensional, or nonsmooth reachability problems, these gradients can be difficult to represent and optimize reliably. 
While architectures such as SIRENs can improve derivative fidelity \cite{bansal2021deepreach}, they often introduce optimization instability and become increasingly difficult to scale in high-capacity settings \cite{raissi2019physics, sitzmann2020implicit}.

A second line of work formulates reachability as a discrete-time zero-sum reinforcement learning (RL) problem \cite{fisac2019bridging,hsu2021safety,hsu2023isaacs,li2025certifiabledeeplearningreachability}. This perspective is attractive because it naturally handles learned dynamics, complex simulators, and high-dimensional systems without requiring direct access to continuous-time PDE derivatives. However, RL-based reachability methods suffer from a fundamental moving-target problem: value estimates, control policies, and adversarial disturbance policies are all updated simultaneously through mutually dependent supervision signals. As the controller adapts to the current disturbance policy, the disturbance simultaneously adapts to the evolving controller, while both remain coupled to a changing value function estimate. Without a stable boundary-conditioned anchor, these recursive updates can accumulate bootstrap error over long horizons, leading to unstable optimization and inaccurate safety boundary propagation.

To address these challenges, we propose a discrete-time, windowed reachability learning framework for computing backward reachable tubes (BRTs) or backward reach-avoid tubes (BRATs). The method propagates safety values backward from the terminal boundary through a temporal curriculum, decomposing the long-horizon problem into sequential local learning stages. Within each temporal window, approximate optimal policies are constructed from the current value estimate through gradient-free bang-bang probing and refined through supervised imitation learning. The value function is then updated using Bellman-Isaacs supervision from both teacher-based one-step targets and student-policy rollout targets. To reduce long-horizon bootstrap drift, each temporal segment is further refined through rollout-anchored boundary correction before serving as the boundary condition for the preceding window. Collectively, this formulation enables scalable reachability analysis and stable long-horizon value propagation for high-dimensional systems.

Our approach can be viewed as combining the most useful aspects of PINN- and RL-based formulations while avoiding their primary failure modes. Akin to PINN-based methods, we rely on self-supervision from the governing optimality equation: instead of minimizing a continuous-time HJI PDE residual, we minimize the discrete-time Bellman-Isaacs error induced by the dynamic programming recursion. This preserves the physics-informed learning signal from PINN-style methods, but avoids their dependence on accurate spatial value gradients, which are difficult to approximate in high-dimensional, non-smooth problems. 
At the same time, our formulation inherits the flexibility of discrete-time RL-style methods, which can be applied to complex sampled (black-box) dynamics. However, rather than learning control and disturbance policies through coupled adversarial actor-critic updates, we exploit the bang-bang structure of the optimal policies to convert actor learning into a supervised learning problem. These supervised policies, together with the terminal boundary condition, provide stable anchors for Bellman target construction, making value propagation substantially more stable than unconstrained minimax policy-value optimization.
Through several simulated case studies and real-world robotic experiments, we demonstrate that our method learns accurate BRT and BRAT value functions in high-dimensional adversarial systems, improves training stability over existing learning-based reachability solvers, and enables robust closed-loop safety-critical control under disturbances.

%% file: formulation.tex
\vspace{-0.5em}
\section{Problem Formulation}
\vspace{-0.5em}
Consider a discrete-time nonlinear system with control-disturbance affine dynamics:
\begin{equation}
    x_{k+1} = f(x_k, u_k, d_k) = f_0(x_k) + f_u(x_k)u_k + f_d(x_k)d_k,
\end{equation}
where $x_k \in \mathcal{X}$, $u_k \in \mathcal{U}= \{ u \in \mathbb{R}^{n_u} \mid \underline{u} \le u \le \overline{u} \}$, and $d_k \in \mathcal{D}= \{ d \in \mathbb{R}^{n_d} \mid \underline{d} \le d \le \overline{d} \}$ denote the state, control, and disturbance at time step $k$, respectively. The disturbance $d_k$ is treated adversarially to capture worst-case safety margins, representing both exogenous environmental disturbances, such as wind gusts or adversarial agents, and internal epistemic uncertainties, such as unmodeled friction. This control-disturbance affine setting remains expressive enough to model a broad class of robotic systems, including autonomous vehicles, drones, and manipulators.

Let the failure set be $\mathcal{L} = \{x \mid \ell(x) \leq 0\}$ and the target set be $\mathcal{G} = \{x \mid g(x) \leq 0\}$. 
We study two closely related finite-horizon reachability problems. The first is an avoid-only safety problem, characterized by a backward reachable tube (BRT), denoted $\mathcal{B}_S(k)$. BRT contains all states at time step $k$ from which failure is inevitable under worst-case disturbances. The second is a reach-avoid problem, characterized by a backward reach-avoid tube (BRAT), denoted $\mathcal{B}_L(k)$. BRAT consists of states from which the target can be reached while avoiding safety violations despite worst-case disturbances. Mathematically,
\begin{equation}
\begin{aligned}
    \mathcal{B}_S(k)&=\left\{x \mid \forall u(\cdot), \exists d(\cdot), \exists \kappa \in[k, K], x_{x, k}^{u, d}(\kappa) \in \mathcal{L}\right\}, \\ 
    \mathcal{B}_L(k)&=\left\{x \mid \forall d(\cdot), \exists u(\cdot), \exists \kappa \in[k, K], x_{x, k}^{u, d}(\kappa) \in \mathcal{G}, \forall s \in [k,\kappa], x_{x, k}^{u, d}(s) \notin \mathcal{L} \right\}, \\ 
\end{aligned}
\end{equation}
where $x_{x,k}^{u,d}$ denotes the trajectory induced by policies $u(\cdot)$ and $d(\cdot)$ starting from state $x$ at time step $k$. For brevity, the remainder of the main text focuses primarily on the BRT formulation. The full BRAT formulation and corresponding Bellman-Isaacs recursions are provided in Appendix~\ref{app:brat}.

For a given task, we synthesize the relevant tube by learning the corresponding value function. 
In avoid-only tasks, we learn the BRT value function $V_S^*$ and its associated safety-preserving policy. The value function is given as \cite{bansal2017hamilton}:
\begin{equation}
\begin{aligned}
    V_S^*(x, k) &= \max_{\mathbf{u}} \min_{\mathbf{d}} \left[ \min_{\kappa \in [k, K]} \ell(x_\kappa)  \right]. \\
\end{aligned}
\end{equation}
Under this convention, $\mathcal{B}_S(k) = \{ x: V_S^*(x, k)\leq 0 \}$. The optimal value function satisfies the following discrete-time Bellman-Isaacs recursions \cite{hsu2023isaacs}:
\begin{equation}
\label{eq:bellman_recursion}
\begin{aligned}
    V_S^*(x, k) &= \min \left(\ell(x), \max_{u \in \mathcal{U}} \min_{d \in \mathcal{D}} V_S^*(f(x, u, d), k+1) \right),\\
\end{aligned}
\end{equation}
with the boundary condition (BC) $b(x)$: $V^*_S(x, K) = \ell(x)$.

Our objective is to approximate these Bellman-Isaacs solutions with neural value functions and their corresponding optimal policies. Specifically, we aim to learn a neural value function $V_\theta(x,k)$ that approximates $V_S^*(x,k)$, thereby providing a quantitative certificate of safety or safe reachability over the state space. In parallel, we learn shared-network control and disturbance policies, $\pi_\phi^u(x,k)$ and $\pi_\phi^d(x,k)$, though our framework readily accommodates separate network architectures.

%% file: method.tex
\vspace{-0.5em}
\section{Method}
\label{sec:method}
\vspace{-0.5em}
Our method learns the neural value function associated with BRT or BRAT. Our key idea is to train backward from the terminal boundary using a windowed temporal curriculum: each window learns a local value function and corresponding control and disturbance policies through approximate teacher supervision and Bellman-error minimization, and then undergoes a finetuning phase for reducing any learning errors before being frozen as the boundary condition for the preceding window.

\emph{\textbf{Temporal Partitioning via Windowed Network Architectures.}} Approximating long-horizon Bellman-Isaacs solutions with a single monolithic network can lead to representational overload, particularly when the value function exhibits sharp temporal variation or nonsmooth reachable-set boundaries. 
To reduce this burden, we partition the discrete horizon $K$ into sequential temporal windows, each represented by a specialized network.
Let $\mathcal{K}=\{0,1,\dots,K\}$. 
We define a neural value function $V_\theta: \mathcal{X} \times \mathcal{K} \to \mathbb{R}$ together with a shared policy network $\pi_\phi^{u,d}: \mathcal{X} \times \mathcal{K} \to \mathcal{U} \times \mathcal{D}$ that parameterizes both the control and disturbance policies.
The value function is boundary-aware by construction at the terminal time $k=K$ and uses independent network weights within each temporal window:
\begin{equation}
V_\theta(x,k) = \begin{cases}
b(x), & \text{if } k=K, \\
V_\theta^w(x,k), & \text{if } K - wW \leq k < K - (w-1)W,
\end{cases}
\label{eq:windowed_value}
\end{equation}
where $W\in\mathbb{Z}^+$ denotes the window length, $w\in\mathbb{Z}^+$ is the window index counted backward from the terminal time, and $b(x)$ is the terminal boundary condition.

The shared policy parameterization is decomposed into temporal windows:
$\pi_\phi^{(\cdot)}(x,k)=\pi_\phi^{(\cdot),w}(x,k)$
for $K-wW \leq k < K-(w-1)W$,
where $(\cdot)\in\{u,d\}$ denotes the control and disturbance policies.
This windowed representation decomposes the long-horizon dynamic programming problem into shorter-horizon subproblems. Each network focuses its representational capacity on a localized temporal segment, while later-time windows, once trained and frozen, serve as fixed boundary conditions for earlier windows.

\begin{figure*}[t]
\centering
  \begin{minipage}[t]{0.68\linewidth}
    \vspace{0pt} 
    \centering
    \includegraphics[width=0.99\linewidth]{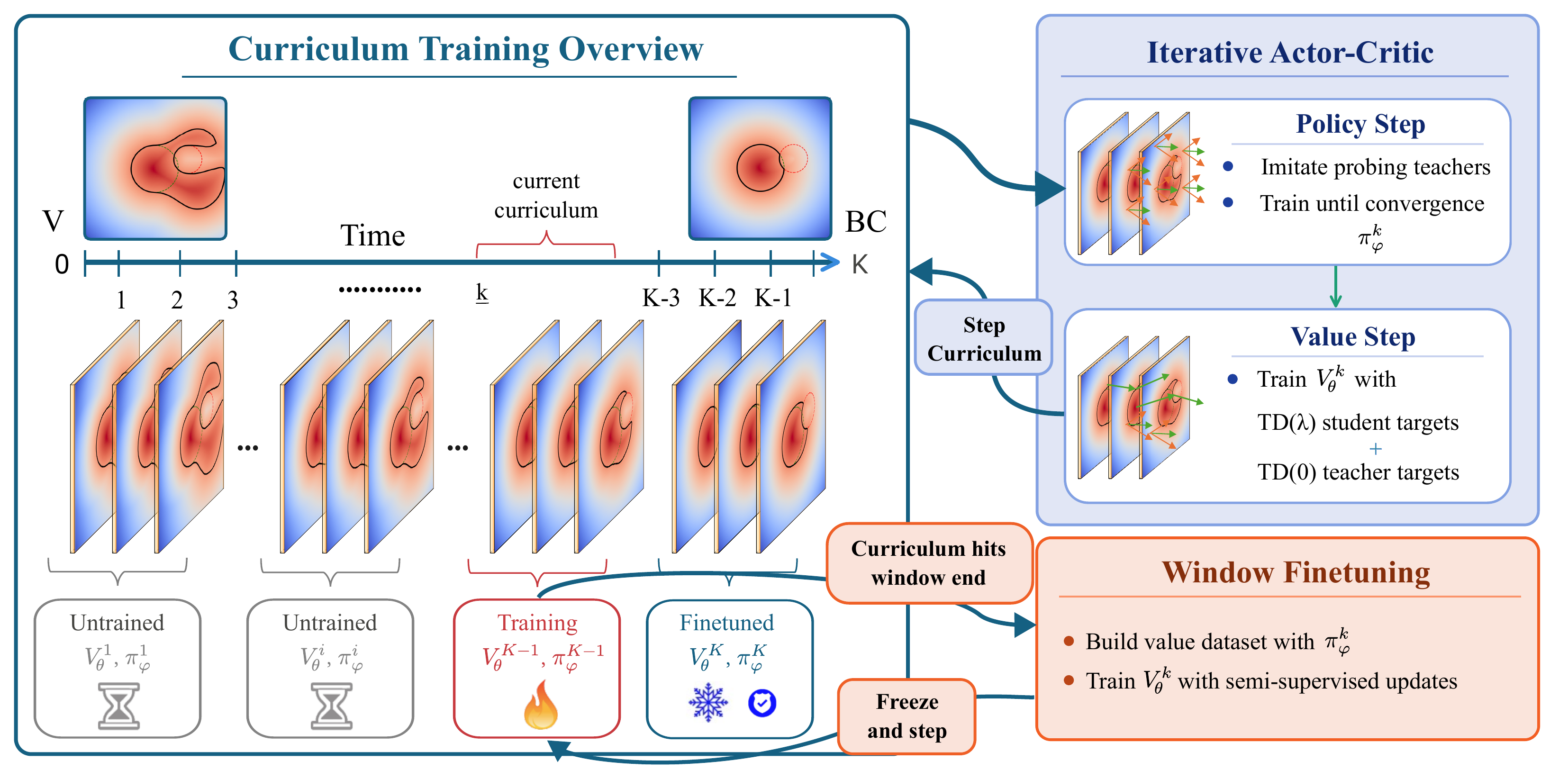}
    \vspace{-3mm}
    
  \end{minipage}
  \hfill
  \begin{minipage}[t]{0.3\linewidth}
    \vspace{-0.8em}
    \begin{algorithm}[H]
    \algrenewcommand{\algorithmiccomment}[1]{\hskip3em$\rightarrow$ #1}
        \caption{Overview}
        \label{alg:training_overview}
        \begin{algorithmic}
        \Require $K$, $g$, $\ell$, $W$
        \State $\underline{k} \gets K$, initialize NNs
        \While{$\underline{k} \geq 0$}
        \State $\pi \gets \text{Supervise}(\hat{u}^{\underline{k}:K}_{\text{teacher}})$
        \State $V \gets \text{Update}(y_{\text{mixed}}^{\underline{k}:K})$
        \If{{\small $(K-\underline{k}) \text{ mod } W = 0$}}
             \State $\text{FinetunePhase}$
             \State Freeze current window
             \State Unfreeze next window
        \EndIf
        \State $\underline{k} \gets \underline{k}- 1$
        \EndWhile
    \end{algorithmic}
    \end{algorithm}
  \end{minipage}
  \caption{Method Overview: backward-temporal curriculum with actor-critic learning.}
  \label{fig:training_overview}
  \vspace{-1.5em}
\end{figure*}
\emph{\textbf{Backward Curriculum.}} As illustrated in Alg.~\ref{alg:training_overview} and Fig.~\ref{fig:training_overview}, training proceeds backward along the discrete time axis, starting from the terminal boundary condition.
At each curriculum step, the algorithm alternates between policy training and value training, followed by a window-finetuning phase whenever a temporal window is completed. Each phase is trained to near convergence before the curriculum advances.
Crucially, the policy update is performed before the value update to ensure that the updated policies are available to construct accurate temporal-difference (TD) targets for value learning. Whenever the elapsed curriculum horizon $(K - \underline{k})$ reaches a multiple of the window duration $W$, the algorithm triggers a finetuning phase to calibrate the value network against rollout-derived targets, effectively aligning the policies and the value function.
After fine-tuning, the parameters of the completed window, including both value and policy networks, are frozen and used as the boundary condition for the preceding window. 
We now detail each of these phases.

\vspace{-0.5em}
\subsection{Policy Training via Bang-Bang Probing}
\vspace{-0.5em}
The policy training phase learns control and disturbance policies that approximate the optimal actions in the Bellman-Isaacs recursion in Eqn.~\ref{eq:bellman_recursion}.
To mitigate the ``moving target'' problem prevalent in joint control-disturbance optimization (a minimax update), 
we construct approximate teacher actions from the current value estimate and use them to supervise the student policy networks:
\begin{equation}
\mathcal{L}_{\pi} = \frac{1}{N} \sum_{i=1}^{N} w_i\left(\left\|\pi_{\phi}^{u}(x_i,k_i)-\hat{u}^{*}(x_i,k_i) \right\|_2^2 + \left\|\pi_{\phi}^{d}(x_i,k_i)-\hat{d}^{*}(x_i,k_i)\right\|_2^2\right),
\label{eq:policy_obj}
\end{equation}
where $x_i \in \mathcal{X}$ are sampled states, $k_i \in \mathcal{K}_{\text{curr}} = \{\underline{k}, \underline{k}+1, \dots, K-(w-1)W-1\}$ are time indices, $(\hat{u}^*,\hat{d}^*)$ are the estimated teacher actions, and $w_i$ weights the confidence of each teacher label.
\paragraph{Estimating Teacher Policies.}
Our teacher construction is inspired by the bang-bang structure of optimal policies in continuous-time Hamilton-Jacobi reachability \cite{bansal2017hamilton, mitchell2007toolbox}. For continuous-time control-disturbance-affine dynamics
$\dot{x} = f_c(x,u,d) = f_{c,0}(x) + f_{c,u}(x)u + f_{c,d}(x)d$, the continuous-time optimal actions are strictly bang-bang:
\begin{equation}
\label{eq:ct_opt_policy}
\hat{u}_j^* =
\begin{cases}
\overline{u}_j, & C_{u,j} > 0, \\
\underline{u}_j, & C_{u,j} \leq 0,
\end{cases}
\qquad
\hat{d}_j^* =
\begin{cases}
\underline{d}_j, & C_{d,j} > 0, \\
\overline{d}_j, & C_{d,j} \leq 0, 
\end{cases}
\end{equation}
where $C_{u,j}$ and $C_{d,j}$ are the $j$-th components of the gradient projections $f_{c,u}(x)^\top \nabla_x V$ and $f_{c,d}(x)^\top \nabla_x V$, respectively.
For the opposite minimization-maximization convention (e.g., in BRAT), the endpoint assignments are reversed accordingly.

Unlike PINN-based reachability solvers, our method does not require analytic spatial gradients $\nabla_x V_\theta$ for policy extraction. Instead, we infer the relevant gradient-projection signs through discrete value comparisons of forward-simulated next states. For each control dimension $j$, we evaluate $\Delta V_j^u := V_\theta(x^{\overline{u}_j}, k+1) - V_\theta(x^{\underline{u}_j}, k+1) \approx \Delta t \, \nabla_x V_\theta(x, k+1)^\top f_{c,u}(x) \big(\overline{u}_j - \underline{u}_j\big) = \Delta t \, C_{u,j} \big(\overline{u}_j - \underline{u}_j\big)$,
where $x^{\overline{u}_j}$ and $x^{\underline{u}_j}$ denote the next discrete-time states obtained by setting only the $j$-th control dimension to its upper or lower bound, while holding the remaining action dimensions fixed at a baseline value. 
Since $\overline{u}_j-\underline{u}_j>0$, the sign of $\Delta V_j^u$ approximates the sign of $C_{u,j}$ for sufficiently small discretization step $\Delta t$. An analogous probing procedure yields $\mathrm{sign}(C_{d,j}) \approx \mathrm{sign}(\Delta V_j^d)$ for each disturbance dimension.

This procedure requires only $2(n_u+n_d)$ forward value evaluations per sampled state and provides an efficient, gradient-free approximation of the teacher actions $(\hat{u}^*,\hat{d}^*)$. To reduce the influence of ambiguous labels, such as regions where the value is locally flat or multiple actions yield nearly identical next-step values (e.g., $\nabla_x V  \rightarrow \mathbf{0}$), we downweight uncertain teachers through the confidence weight $w_i$, computed as the inverse variance of these directional probe scores.
\vspace{-0.3em}
\subsection{Value Training via Mixed Temporal-Difference Targets}
\vspace{-0.3em}
A fundamental challenge in this decoupled actor-critic scheme is compounding suboptimality: if a suboptimal student policy is used to train the value network, the resulting value estimate may become biased, which in turn yields corrupted teacher actions in subsequent policy updates. 
To mitigate this error propagation, we train the value function using a mixture of one-step TD targets generated from teacher actions and multi-step TD targets generated from student-policy rollouts. The teacher targets keep the value update anchored to the Bellman-Isaacs recursion to preserve optimality, while the student targets improve consistency with the learned policies over longer rollouts.

For a sampled state-time pair $(x_0,k_0)$, we first compute a one-step teacher target. Let $x_1 = f(x_0,\hat{u}^*,\hat{d}^*)$ be the next state under the teacher actions. The teacher TD target is then
$y_{\text{teacher}} = \mathcal{J}\bigl(x_0, V_\theta(x_1, k_0 + 1)\bigr)$,
where $\mathcal{J}(x,V) = \min\bigl(\ell(x), V\bigr)$ for BRT computations.

We also construct a multi-step student target by rolling out the learned policies
$(\pi_\phi^u,\pi_\phi^d)$ for $M$ steps, where $M$ is sampled from a geometrically decaying distribution. Let $x_{0:M}$ denote the resulting trajectory and let
$V_\theta(x_M,k_0+M)$ be the bootstrapped terminal value. The student target is $y_{\text{student}} = \mathcal{J}\bigl(x_{0:M}, V_\theta(x_M, k_0 + M)\bigr)$.
For BRT, this reduces to the minimum safety margin along the rollout, clipped by the bootstrapped value at the terminal rollout state. 

The value network is trained to minimize the MSE residual of $V_\theta$ against the evenly mixed targets (i.e., 50\% student samples). 
To improve boundary-condition compliance, we oversample time indices near the current boundary, $k_0=K-(w-1)W-1$, when drawing $k_0 \in \mathcal{K}_{\text{curr}}$. 
\vspace{-0.3em}
\subsection{Finetune Phase}
\vspace{-0.3em}
Although mixed TD targets reduce short-horizon suboptimality, the backward windowed curriculum remains susceptible to bootstrap drift across temporal boundaries. If a completed window is frozen with biased value estimates, those errors become the boundary condition for the preceding window and can propagate backward through the horizon. To reduce this drift, we perform a boundary-correction fine-tuning phase at the end of each temporal window before freezing its parameters.

This phase calibrates the value network against rollout-derived targets under the \textit{learned} student policies. Specifically, we optimize a semi-supervised objective that combines local Bellman consistency with supervised anchor targets obtained from discrete-time dynamics rollouts:

\begin{equation}
\label{eq:ft_loss}
\resizebox{\linewidth}{!}{$
\mathcal{L}_{\text{FT}} = \mathcal{L}_{\text{TD}}^{\text{student}} + \lambda_{\text{sup}} \mathcal{L}_{\text{sup}} = \frac{1}{B} \left( \sum_{i=1}^{B} \left( V_\theta(x_i,k_i) - y_{\text{student},i} \right)^2 + \lambda_{\text{sup}} \sum_{m=1}^{B} w_m \left( V_\theta(x_m,k_m) - G_m \right)^2 \right).
$}
\end{equation}

Here, $\mathcal{L}_{\mathrm{TD}}^{\mathrm{student}}$ uses only student-policy TD targets, omitting teacher targets so that the update aligns the value network with the policies being evaluated. The second term anchors the value network to empirical rollout targets $G_m$, sampled from a static dataset.
To construct the $G_m$ dataset, we forward-simulate from $(x_m,k_m)$ using the learned student policies until reaching the upper boundary of the active window, $k_{\text{bound}}=K-(w-1)W$, where $w\geq 1$ is the current window index. 
The anchor target is computed from the same stage-wise reachability cost along this rollout and bootstrapped with the frozen value function at $k_{\mathrm{bound}}$: $G^{\text{BRT}}(x,k) = \mathcal{J}\bigl(x_{k:k_{\mathrm{bound}}}, V_\theta(x_{k_{\mathrm{bound}}}, k_{\mathrm{bound}})\bigr)=  \min \left[ \min_{\kappa \in [k, k_{\text{bound}}]} \ell(x_\kappa), \; V_\theta \left( x_{k_{\text{bound}}}, k_{\text{bound}}\right) \right].$
The first term evaluates whether the target is reached safely within the active window, while the second term bootstraps from the frozen boundary value.
To discourage false-positive safety predictions, we apply asymmetric weights
$w_m = 1 + \lambda_{\text{fp}} \mathds{1}\{V_\theta(x_m,k_m) > 0 \text{ and } G_m < 0\}$,
where $\lambda_{\mathrm{fp}}$ controls the penalty for states predicted to be safe but evaluated as unsafe by rollout. The fine-tuning phase is run with a reduced learning rate until convergence. The active window is then frozen and used as the boundary condition for the preceding window.

%% file: results.tex
\vspace{-0.5em}
\section{Experiment Results}
\vspace{-0.5em}
\label{sec:results}
We evaluate our framework across three benchmarks spanning distinct types of reachability problems: (i) \textbf{Two-Vehicle Narrow Passage} \cite{bansal2021deepreach} (Reach-Avoid); (ii) \textbf{Quadrotor Gate Avoidance} (Avoid); and (iii) \textbf{40D/80D Publisher-Subscriber} \cite{pmlr-v283-sharpless25a} (Reach). We compare against three state-of-the-art baselines: \textbf{DeepReachMPC} \cite{feng2025bridging} (a monolithic continuous-time PINN-based solver), \textbf{RARL} \cite{hsu2021safety} (an RL-based solver for BRAT, i.e., Narrow Passage), and \textbf{ISAACS} \cite{hsu2023isaacs} (an actor-critic RL baseline for BRT, namely the quadrotor and the publisher-subscriber problems). All baselines are implemented using their official codebases with careful parameter tuning. Detailed training configurations and environment setups are deferred to Appendix \ref{app:experimental_setup}.

\textbf{Evaluation metrics.} We track the \textbf{False Positive Rate (FPR)} (predicted safe but empirical failure, indicating dangerous over-optimism) and \textbf{False Negative Rate (FNR)} (predicted unsafe but empirical success, indicating over-conservatism) of the value function. We further report \textbf{Success Rate} to assess policy optimality, and \textbf{Training Time} to evaluate computational efficiency. While FPR and FNR measure the consistency between the learned value and policy, and the success rate measures the quality of the learned policy, neither enables a fair comparison of the learned safety certificates across methods. We therefore additionally report the \textbf{Learned Volume} (the fraction of the evaluated state space that the value function claims safe) and the \textbf{Verified Volume} \cite{pmlr-v242-lin24a}: the fraction of the same evaluated state space covered by a $\delta$-level subset of the learned safe set that is certified to have a safety violation rate of at most $\epsilon=0.01$ with confidence $1-\beta$, $\beta=10^{-10}$. Each metric is evaluated over 100K samples across 5 random seeds.

\begin{table*}[t]
\centering
\caption{\small Quantitative comparison of open-loop safety and closed-loop control performance. Results report mean and standard deviation over 5 random seeds. For Narrow Passage, we further report the verified volume over the entire state space. FNR is N/A when no rollout succeeds. The reduction in training time for our method is primarily attributable to its implementation in JAX.}
\label{tab:nav_benchmarks}
\resizebox{\textwidth}{!}{
\begin{tabular}{l cccc cc cc}
\toprule
\multirow{2}{*}{\textbf{Method}}
& \multicolumn{4}{c}{\textbf{Narrow Passage}}
& \multicolumn{2}{c}{\textbf{State Space of Interest}}
& \multicolumn{2}{c}{\textbf{Entire State Space}} \\
\cmidrule(lr){2-5} \cmidrule(lr){6-7} \cmidrule(lr){8-9}

& \textbf{FPR \%} $\downarrow$
& \textbf{FNR \%} $\downarrow$
& \textbf{Succ \%} $\uparrow$
& \textbf{Time (h)} $\downarrow$
& \textbf{Learned Vol \%}
& \textbf{Verified Vol \%} $\uparrow$
& \textbf{Learned Vol \%}
& \textbf{Verified Vol \%} $\uparrow$ \\
\midrule

DeepReachMPC
& $\mathbf{0.00 \pm 0.00}$
& $100.00 \pm 0.00$
& $2.50 \pm 0.90$
& $5$
& $0.00 \pm 0.00$
& $0.00 \pm 0.00$
& $0.67 \pm 0.21$
& $0.00 \pm 0.00$ \\

RL Baseline
& $99.97 \pm 0.01$
& $\mathbf{0.00 \pm 0.00}$
& $3.52 \pm 0.91$
& $6$
& $99.97 \pm 0.02$
& $0.00 \pm 0.00$
& $13.19 \pm 0.31$
& $0.00 \pm 0.00$ \\

\textbf{Ours}
& $100.00 \pm 0.00$
& $\mathbf{0.00 \pm 0.00}$
& $\mathbf{99.71 \pm 0.18}$
& $\mathbf{1}$
& $100.00 \pm 0.00$
& $\mathbf{100.00 \pm 0.00}$
& $4.99 \pm 0.02$
& $\mathbf{3.76 \pm 0.07}$ \\
\midrule

\multirow{2}{*}{\textbf{Method}}
& \multicolumn{4}{c}{\textbf{Quadrotor Gate Avoidance}}
& \multicolumn{2}{c}{\textbf{Entire State Space}}
& & \\
\cmidrule(lr){2-5} \cmidrule(lr){6-7}

& \textbf{FPR \%} $\downarrow$
& \textbf{FNR \%} $\downarrow$
& \textbf{Succ \%} $\uparrow$
& \textbf{Time (h)} $\downarrow$
& \textbf{Learned Vol \%}
& \textbf{Verified Vol \%} $\uparrow$
& & \\
\midrule

DeepReachMPC
& $1.95 \pm 0.80$
& $78.40 \pm 14.00$
& $0.63 \pm 0.07$
& $4$
& $2.06 \pm 0.90$
& $0.00 \pm 0.00$
& & \\

RL Baseline
& $67.67 \pm 2.73$
& N/A
& $0.00 \pm 0.00$
& $6$
& $67.67 \pm 2.75$
& $0.00 \pm 0.00$
& & \\

\textbf{Ours}
& $\mathbf{0.14 \pm 0.03}$
& $\mathbf{20.44 \pm 1.44}$
& $\mathbf{4.82 \pm 0.11}$
& $\mathbf{1}$
& $3.73 \pm 0.17$
& $\mathbf{3.65 \pm 0.13}$
& & \\
\bottomrule
\end{tabular}}
\vspace{-1em}
\end{table*}

\begin{table*}[t]
\centering
\caption{\small Scalability stress-testing on high-dimensional Publisher-Subscriber network dynamics. Since this system admits a ground-truth solution, we directly compute FPR, FNR, and MSE against the ground truth.}
\label{tab:network_benchmarks}
\resizebox{\textwidth}{!}{
\begin{tabular}{l ccc ccc}
\toprule
\multirow{2}{*}{\textbf{Method}} & \multicolumn{3}{c}{\textbf{40D Network Dynamics}} & \multicolumn{3}{c}{\textbf{80D Network Dynamics}} \\
\cmidrule(lr){2-4} \cmidrule(lr){5-7}
 & \textbf{FPR \%} $\downarrow$ & \textbf{FNR \%} $\downarrow$ & \textbf{MSE} $\downarrow$ & \textbf{FPR \%} $\downarrow$ & \textbf{FNR \%} $\downarrow$ & \textbf{MSE} $\downarrow$ \\
\midrule
DeepReachMPC & $2.648 \pm 0.310$ & $\mathbf{0.001 \pm 0.001}$ & $0.007 \pm 0.001$ & $2.115 \pm 0.145$ & $\mathbf{0.000 \pm 0.000}$ & $0.026 \pm 0.001$ \\
RL Baseline  & $23.306 \pm 0.397$ & $76.228 \pm 0.602$ & $1.717 \pm 0.086$ & $25.581 \pm 0.941$ & $78.696 \pm 0.731$ & $5.886 \pm 0.248$ \\
\textbf{Ours} & $\mathbf{2.319 \pm 0.207}$ & $0.032 \pm 0.011$ & $\mathbf{0.006 \pm 0.001}$ & $\mathbf{2.099 \pm 0.382}$ & $0.005 \pm 0.010$ & $\mathbf{0.023 \pm 0.008}$ \\
\bottomrule
\end{tabular}
}
\vspace{-1em}
\end{table*}

\emph{\textbf{Long-Horizon Narrow Passage.}}
Two opposing vehicles are tasked with safely reaching their respective goals through a bottleneck caused by a stranded vehicle. To avoid a head-on collision, one agent must proactively yield early in the extended $T=8.0$~s horizon. This requires the reachability solver to robustly maintain the early-time optimal value function, preventing agents from adopting myopic maneuvers that greedily approach the goal at the cost of mild safety violations. 
\begin{wrapfigure}{r}{0.45\textwidth}
  \vspace{-10pt}
  \centering
  \includegraphics[width=\linewidth]{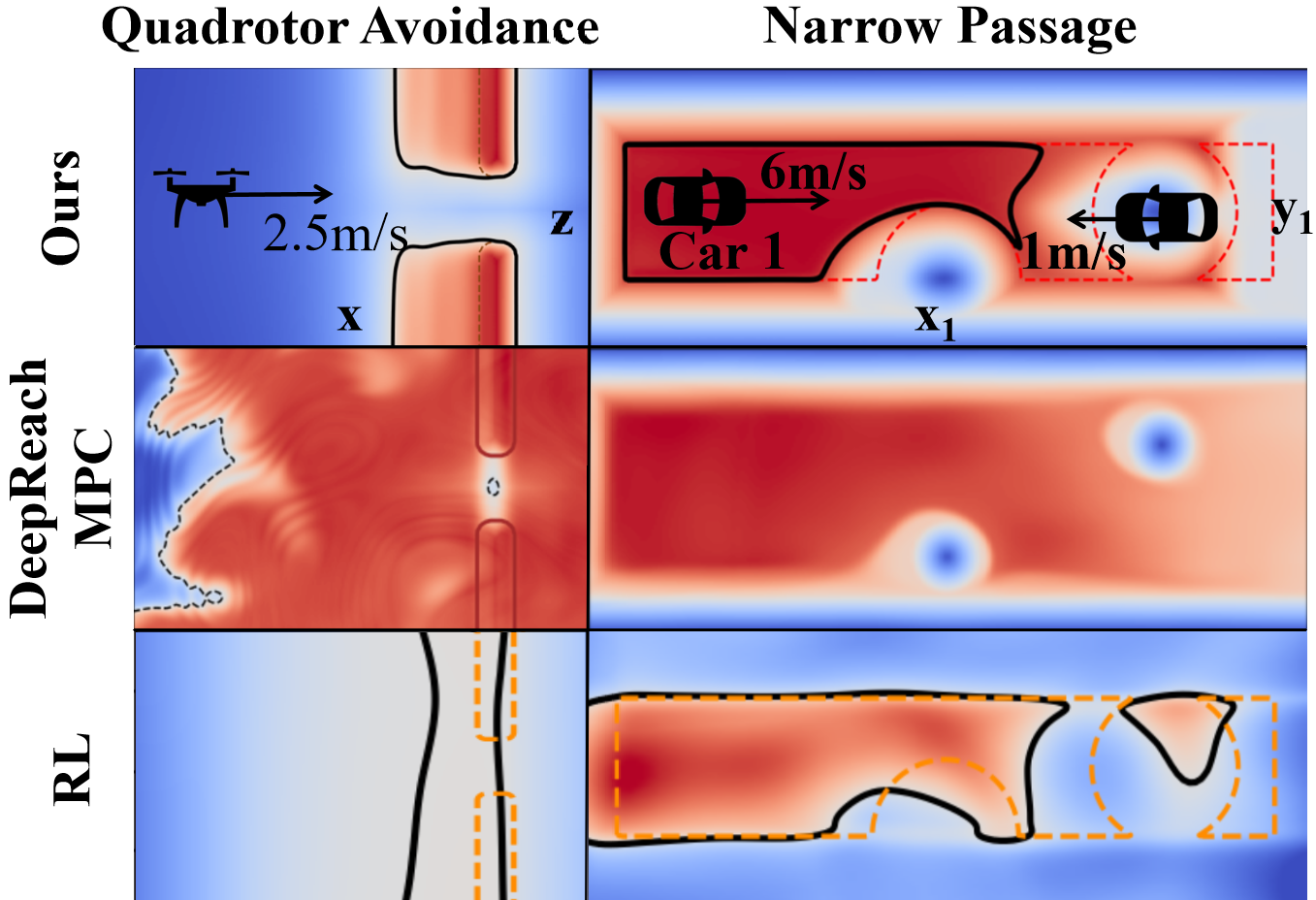}
  \caption{\small Learned value function slices. Quadrotor: $x$-$z$ plane with $v_x=2.5$ and all other states at hover. Narrow Passage: $x$-$y$ plane of car 1 at $T=8$~s, with $[\theta_1, v_1, \phi_1, x_2, y_2, \theta_2, v_2, \phi_2]=[0, 6, 0, 8, 0, -\pi, 1, 0]$. Only our method recovers a valid safe set in both tasks, whereas DeepReachMPC and the RL baseline fail to accurately propagate the value function over long horizons or against complex, non-smooth safety boundaries.}
  \label{fig:heatmaps}
  \vspace{-10pt}
\end{wrapfigure}
As shown in Table~\ref{tab:nav_benchmarks}, our method successfully learns these far-sighted coordination strategies by leveraging temporal partitioning to preserve short-horizon intermediate solutions required for safe traversal.
Notably, we initialize the evaluation trajectories within the state space of interest, where cars start at opposite ends facing their goals at low speed, reflecting typical traffic interactions.
These states are safe by construction, providing a reliable evaluation without ground-truth reachability labels.
Our policy achieves near-perfect success: only 286 failures among 100K rollouts. However, because there are no true negatives, these 286 FPs mathematically yield an FPR of 286/(286+0)=100\%.
Hence, the reported FPR reflects policy execution flaws over the long horizon, rather than a severely over-optimistic value function.
Consistently, the verification certifies our learned safe set over the entire state space of interest (100\% verified volume), whereas no subset of either baseline's learned safe set can be verified over the state space of interest or over the entire state space. Over the entire state space, our method certifies $3.76\%$ of all states, which is consistent with the $5.32\%$ policy success rate. This volume is small because most uniformly sampled states either cannot reach the goal within the horizon (e.g., cars facing away from their goals) or are doomed to fail even under optimal actions (e.g., cars heading toward an obstacle at high speed). 

\emph{\textbf{13D Quadrotor Gate Avoidance.}}
A quadrotor must safely navigate through a gate opening within an otherwise infinite $y$-$z$ planar obstacle. The vehicle must also maintain a positive $x$-velocity while adhering to strict linear and angular velocity bounds across all axes. The system's high thrust-to-mass ratio of 20, combined with the non-smooth safety boundaries, introduces severe numerical instability for standard learning-based solvers, especially those utilizing explicit gradient representations during training. In contrast, our gradient-free approach leverages the stable anchors provided by the temporal curriculum and empirical rollouts to robustly propagate the value function. Crucially, while our absolute 4.8\% success rate appears low, it reflects the physical reality of the environment: the vast majority of random initial states are dynamically doomed to violate the $z$-velocity constraint due to gravity. We further validate the optimality of our learned policy by outperforming both the built-in sampling-based MPC solver of DeepReachMPC (0.3\%) and a carefully tuned receding-horizon MPPI controller (2.1\%). Our method is also the only one that yields a non-empty verified safe volume.
Fig.~\ref{fig:heatmaps} illustrates this gap qualitatively: the sinusoidal network of DeepReachMPC induces severe rippling artifacts when learning the value function against the complex, non-smooth safety boundary, whereas our method yields an intuitively valid safe set. Similarly, RL methods fail to resolve the safe-set geometries due to their relatively weak training signals.

\emph{\textbf{Publisher-Subscriber Systems.}}
This benchmark stress-tests scalability against massive state and action spaces, where the control dimension is exactly the state dimension minus one (representing a single uncontrollable publisher and $N-1$ controllable subscribers). Because the dynamics are structurally decomposable, the system admits a ground-truth solution for the MSE evaluation. As shown in Table~\ref{tab:network_benchmarks}, our framework---despite being formulated as a generalized solver for two-player zero-sum games---achieves an MSE on par with DeepReachMPC, which is engineered for single-player optimal control. This demonstrates that our method scales gracefully to high-dimensional domains.

\emph{\textbf{Ablation on Policy Optimization.}}
We isolate and validate our critical architectural choices using the long-horizon Narrow Passage dynamics as a benchmark. Specifically, we compare our value-grounded teacher-student objective against policy gradient (PG) alternatives, which update the policy networks using TD targets derived from policy rollouts whose lengths are sampled from a geometric distribution between 1 and $m$ steps. As shown in Table~\ref{tab:pg_il}, the teacher-student formulation achieves faster convergence and improved consistency across random seeds. By deriving supervision directly from the extremal-action decisions induced by $V_\theta$ (Eqn.~\ref{eq:policy_obj}), our approach is significantly more robust to the sparse gradient signals inherent to reachability (where trajectory cost depends solely on the extremum state, e.g., $\min_{\kappa} \ell(x(\kappa))$). Additionally, the teacher-student paradigm eliminates the need to tune the maximum rollout horizon ($m$), a hyperparameter that heavily dictates the empirical stability of PG-based training. This highlights that our framework maintains the architectural benefits of actor-critic methods while explicitly avoiding the unstable actor-training dynamics inherent to traditional RL. 

Due to space constraints, we defer additional ablations on the teacher-mixing ratio to Appendix~\ref{app:teacher_mixing} and benchmarks on two-player zero-sum games to Appendix~\ref{app:zero_sum}.
\begin{table}[t]
\centering
\small
\setlength{\tabcolsep}{4pt}
\caption{\small Ablation of policy optimization objectives on the Narrow Passage benchmark.}
\label{tab:pg_il}
\begin{tabular}{l cccc}
\toprule
\textbf{Configuration} & \textbf{FPR \%} $\downarrow$ & \textbf{FNR \%} $\downarrow$ & \textbf{Succ \%} $\uparrow$ & \textbf{Time (mins)} $\downarrow$ \\
\midrule
Policy Gradient - 1 Step & $99.58 \pm 0.09$ & $0.79 \pm 0.74$ & $92.20 \pm 6.11$ & $96.8$\\        
Policy Gradient - 3 Step & $\mathbf{93.57 \pm 8.71}$ & $0.25 \pm 0.13$ & $90.15 \pm 6.87$ & $140.1$ \\
Policy Gradient - 5 Step & $99.75 \pm 0.09$ & $0.28 \pm 0.08$ & $80.81 \pm 3.85$ & $210.6$\\
\textbf{Teacher Student} & $100.00 \pm 0.00$ & $\mathbf{0.00 \pm 0.00}$ & $\mathbf{99.71 \pm 0.18}$ & $81.3$\\
\bottomrule
\end{tabular}
\vspace{-1.5em}
\end{table}
\vspace{-0.5em}
\subsection{Vision-Based F1Tenth Racing and Sim-to-Real Deployment}
\vspace{-0.5em}
\label{subsec:racing}
We scale our framework to synthesize a safety value function for autonomous F1Tenth racing with a max speed of 10~m/s, where the vehicle must remain collision-free against the track boundaries with acceleration and steering rate as control inputs. We train the reachability solution jointly across 20 tracks and evaluate it on 3 held-out tracks from \cite{Betz2022_RacingSurvey}, each spanning roughly 120~m $\times$ 120~m. The model takes as input an egocentric $128 \times 128$ Bird's-Eye View (BEV) binary occupancy image concatenated with the vehicle's proprioceptive state (speed, angular velocity, slip angle, steering angle) and a disturbance scale parameter $s_d \in [0.0, 0.3]$. Crucially, the BEV spatial resolution dynamically scales with vehicle speed to guarantee sufficient look-ahead visibility for accurate safety predictions. We note that commonly used low-dimensional inputs, such as LiDAR scans, lack sufficient information for accurate safety prediction in a racing context (e.g., due to occlusion near sharp turns). Because the BEV pixel dynamics are unknown, we treat the system as a black box with underlying control-disturbance-affine dynamics. To ensure robustness against OOD observations and modeling errors, we train against adversarial disturbances including lateral drift and control authority degradation. The entire training process takes approximately 1 hour to complete from scratch. At test time, the learned value function is equipped with a Discrete-Time Control Barrier Function (DCBF) filter \cite{agrawal2017discrete} to safeguard a poorly-tuned MPPI nominal controller. Across 50 evaluation trials (40~s each), the filtered policy achieves zero collisions on both the in-distribution and OOD tracks under a disturbance scale $s_d=0.1$, with average speeds of 9.25 and 9.17~m/s, respectively. Without the filter, the MPPI baseline yields 56\% and 48\% collision rates, respectively.

\begin{wrapfigure}{r}{0.49\textwidth}
  \vspace{-10pt} 
  \centering
  \includegraphics[width=\linewidth]{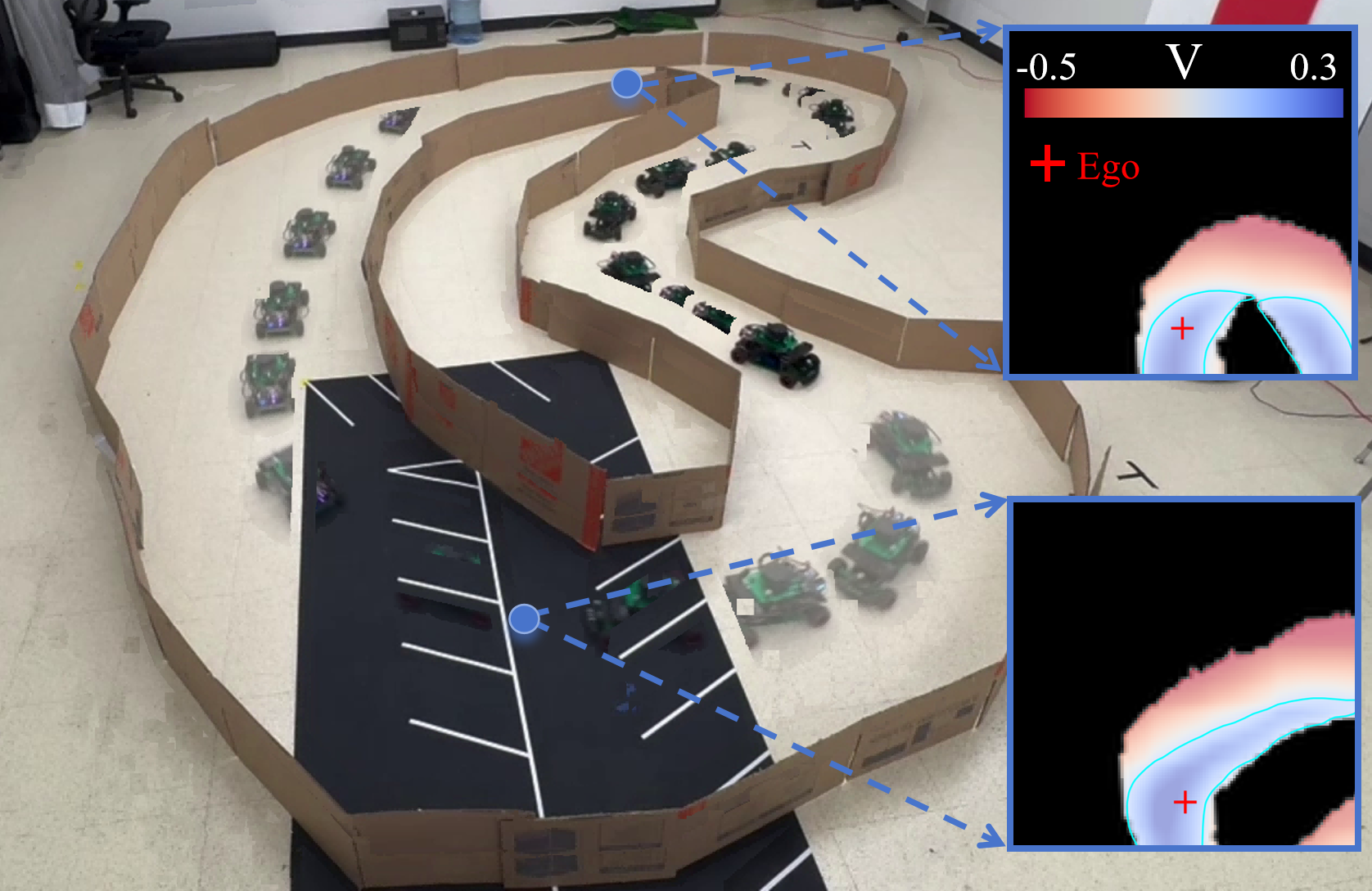}
  \caption{\small Real-world deployment on the Rosmaster R2 RC car, with value function slices visualized at critical states along the trajectory.}
  \label{fig:hardware}
  \vspace{-10pt} 
\end{wrapfigure}
We further finetune our value networks on 20 scaled-down tracks to match real-world dimensions and deploy them onto a physical Rosmaster R2 RC car on a previously unseen track. Following a low-speed warmup lap to construct a track map, we utilize Vicon motion capture for real-time localization, as the 7~Hz onboard LiDAR suffers from severe localization drift at higher speeds. All computations---including the nominal controller, BEV image generation, and safety filtering---execute on an offboard laptop with a GTX 1060 GPU and a 12-core 3.3~GHz CPU. The laptop streams linear and angular reference velocities to the car over WiFi. This closed-loop deployment introduces moderate sim-to-real challenges: a mismatch between the hardware and training control spaces, OOD track geometries, and approximately 50~ms of combined computation and communication latency. By applying a disturbance scale of 0.2, our adversarial formulation enables the safety filter to robustly absorb these modeling errors and discrepancies. Operating at speeds up to 1.5~m/s (hardware physical limit), the filtered policy completed 10 laps with zero collisions and an average speed of 1.07~m/s. In stark contrast, the unfiltered nominal policy, which blindly tracks the racing line with maximal speed, led to 3 collisions within its very first lap. More details are available in Appendix~\ref{app:hardware}.

%% file: conclusion.tex
\vspace{-0.5em}
\section{Conclusion}
\vspace{-0.5em}
\label{sec:conclusion}
In this work, we introduced a discrete-time, gradient-free neural framework for high-dimensional Hamilton-Jacobi (HJ) reachability. Our method partitions long-horizon problems into temporal windows, adapting a teacher-student actor-critic framework to jointly approximate optimal policies and value functions. By inheriting the backward temporal curriculum from PINNs, we preserve stable boundary anchoring while maintaining the flexibility of discrete-time RL. Across diverse safety-critical benchmarks, our framework excels at solving long-horizon and numerically stiff reachability problems. Moreover, it scales to massive state, control, and observation dimensions, and achieves robust sim-to-real generalization through closed-loop physical racing.

%% file: limitations.tex
\vspace{-0.5em}
\section{Limitations and Future Work}
\vspace{-0.5em}
\label{sec:limitations}
While our discrete-time framework scales robustly to high-dimensional systems and vision-based observations, several avenues for refinement remain. 
First, while our method is relatively insensitive to the integration time step (see Appendix~\ref{app:dt_ablation}), it does not converge to the exact optimal solution of discrete-time systems, since it restricts the policy to be bang-bang (see Appendix~\ref{app:optimality_gap}).
An exciting future direction is to replace single-step bang-bang teachers with multi-step episodic rollouts, which would also relax the affine dynamics requirement---at the cost of higher variance in the learning signal---and enable the framework to solve reachability problems that demand more nuanced control, such as dexterous manipulation.
Second, the framework currently samples training data over the entire state space, which scales poorly to high-dimensional, contact-rich tasks.
Hence, enabling data collection through episodic rollouts would be highly impactful.

%% file: appendix.tex
\newpage
\begin{appendices}
The appendix is organized as follows:
\begin{itemize}
    \item \textbf{Appendix A} provides detailed derivations for the Backward Reach-Avoid Tube (BRAT) formulation.
    \item \textbf{Appendix B} outlines extended implementation details.
    \item \textbf{Appendix C} presents a theoretical and empirical discussion on the discrete-time optimality gap of our framework.
    \item \textbf{Appendix D} contains extended ablation studies isolating the impact of our core algorithmic design choices.
    \item \textbf{Appendix E} expands on our empirical evaluation with additional benchmarking results on two-player zero-sum games.
    \item \textbf{Appendix F} details the benchmarking problem formulations and system dynamics.
    \item \textbf{Appendix G} describes the hardware specifications and deployment architecture for the physical RC car safety filter.
\end{itemize}
\section{BRAT Derivations and Descriptions}
\label{app:brat}

As stated in the main text, our proposed discrete-time neural reachability framework generalizes naturally to tasks that require simultaneous goal-reaching and safety tracking. In this appendix, we present the Backward Reach-Avoid Tube (BRAT) formulation in detail, including the corresponding Bellman-Isaacs recursion operators, the TD target, and the rollout reachability anchoring labels utilized during the fine-tuning phase.

\subsection{Reach-Avoid Value Function and Recursion}

Let the failure set be characterized by the zero sub-level set $\mathcal{L} = \{x \mid \ell(x) \leq 0\}$ and the target set by $\mathcal{G} = \{x \mid g(x) \leq 0\}$. The backward reach-avoid tube $\mathcal{B}_L(k)$ represents the set of all initial states at time step $k$ from which a controller can guarantee that the system enters the target set $\mathcal{G}$ within the finite horizon while strictly remaining outside the failure set $\mathcal{L}$ under worst-case disturbances:
\begin{equation}
\mathcal{B}_L(k)=\left\{x \;\middle|\; \forall d(\cdot), \, \exists u(\cdot), \, \exists \kappa \in[k, K] \text{ s.t. } x_{x, k}^{u, d}(\kappa) \in \mathcal{G}, \, \text{and } \forall s \in [k,\kappa], \, x_{x, k}^{u, d}(s) \notin \mathcal{L} \right\}.
\end{equation}

By adopting a minimization-maximization convention over the target distance and historical safety margins, the optimal finite-horizon BRAT value function $V_L^*(x, k)$ is characterized as:
\begin{equation}
V_L^*(x, k) = \min_{\mathbf{u}} \max_{\mathbf{d}} \left[ \min_{\kappa \in [k, K]} \left( \max \left( g(x_\kappa), \max_{s\in[k,\kappa]} -\ell(x_s) \right) \right) \right],
\end{equation}
where the underlying reach-avoid tube is recovered at the zero sub-level set $\mathcal{B}_L(k) = \{ x \mid V_L^*(x, k)\leq 0 \}$. The solution satisfies the following discrete-time Bellman-Isaacs recursion \cite{hsu2023isaacs}:
\begin{equation}
V_L(x, k) = \max \left( -\ell(x), \min \left( g(x), \min_{u \in \mathcal{U}} \max_{d \in \mathcal{D}} V_L(f(x, u, d), k+1) \right) \right),
\end{equation}
subject to the joint terminal boundary condition (BC):
\begin{equation}
V_L(x, K) = \max\big(g(x), -\ell(x)\big).
\end{equation}

\subsection{Policy Extraction and Teacher Adjustments}

To approximate the optimal minimax action pairs without explicit spatial gradients $\nabla_x V_\theta$, our method simply adopts the same bang-bang probing strategy, except with mirrored signs due to the reversed min-max assignment of the control and disturbance. Specifically, 
\begin{equation}
\label{eq:ct_opt_policy_brat}
\hat{u}_j^* =
\begin{cases}
\overline{u}_j, & C_{u,j} \leq 0, \\
\underline{u}_j, & C_{u,j} > 0,
\end{cases}
\qquad
\hat{d}_j^* =
\begin{cases}
\underline{d}_j, & C_{d,j} \leq 0, \\
\overline{d}_j, & C_{d,j} > 0, 
\end{cases}
\end{equation}
where the sign of $C$'s can be similarly estimated by the sign of $\Delta V_j^u := V_\theta(x^{\overline{u}_j}, k+1) - V_\theta(x^{\underline{u}_j}, k+1)$. The same procedure provides $\mathrm{sign}(C_{d,j}) \approx \mathrm{sign}(\Delta V_j^d)$. 
Although this formulation technically breaks the non-anticipative strategy assumption for the disturbance in a strict discrete-time setting, Appendix~\ref{app:zero_sum} empirically shows that for small $\Delta t$, our method remains on par with numerical solvers in two-player games.

\subsection{TD and Fine-Tuning Anchor Targets}
Let the one-step backward reach-avoid operator be defined as:
\begin{equation}
\mathcal{J}_{\text{BRAT}}(x, V) = \max\left(-\ell(x), \min\left(g(x), V\right)\right).
\end{equation}
For a sampled state-time pair $(x_0,k_0)$, the teacher TD target is given by $y_{\text{teacher}} = \mathcal{J}_{\text{BRAT}}\bigl(x_0, V_\theta(x_1, k_0 + 1)\bigr)$, where $x_1 = f(x_0,\hat{u}^*,\hat{d}^*)$. 

The $M$-step student target, $y_{\text{student}}$, evaluates the reach-avoid cost under the learned student policies $\pi_\phi^{u,d}$, either reaching the target within the rollout or deferring to the bootstrapped value at the final step while remaining safe along the rollout:
\begin{equation}
\begin{split}
y_{\text{student}} = \min \Big[ & \min_{\kappa \in [k_0, k_0+M]} \Big( \max \big( g(x_\kappa), \max_{s\in[k_0,\kappa]} -\ell(x_s) \big) \Big), \\
& \max \Big( \max_{s\in[k_0,k_0+M]} -\ell(x_s), \; V_\theta(x_M, k_0+M) \Big) \Big].
\end{split}
\end{equation}

When the curriculum completes a temporal window and triggers the boundary-correction fine-tuning phase, we construct empirical anchor targets, $G^{\text{BRAT}}$. Instead of an $M$-step horizon, we evaluate the rollout until the current window's temporal boundary, $k_{\text{bound}}$, and bootstrap from the frozen downstream value function:
\begin{equation}
\begin{split}
G^{\text{BRAT}} = \min \Big[ & \min_{\kappa \in [k_m, k_{\text{bound}}]} \Big( \max \big( g(x_\kappa), \max_{s\in[k_m,\kappa]} -\ell(x_s) \big) \Big), \\
& \max \Big( \max_{s\in[k_m,k_{\text{bound}}]} -\ell(x_s), \; V_{\theta_{\text{frozen}}} \big( x_{k_{\text{bound}}}, k_{\text{bound}}\big) \Big) \Big].
\end{split}
\end{equation}
To penalize over-optimistic value predictions (false positives), we adapt the asymmetric penalty $w_m$ in Eqn.~\eqref{eq:ft_loss} by mirroring the indicator conditions: $w_m = 1 + \lambda_{\text{fp}} \mathds{1}\{V_\theta(x_m,k_m) \leq 0 \text{ and } G_m > 0\}$.

Upon convergence of this objective, the network parameters are frozen to serve as the new boundary condition for the preceding temporal segment.

\section{Additional Method Details}
This section presents supplementary methodological details that are omitted from the main text for brevity.
\subsection{Sub-Step Safety Evaluation and Tie-Breaking}

While our framework explicitly accounts for discrete-time safety, violations can still occur between discrete time steps. To mitigate this, one can roll out the transition using a finer time resolution to detect safety violations during these intermediate sub-steps.
To incorporate this sub-step safety verification into the probing teacher, we determine the control actions by evaluating the sign of $\Delta J_j^u$ instead of the standard value difference $\Delta V_j^u$ (which corresponds to the sign of $C_{u,j}$ in Eqn.~\ref{eq:ct_opt_policy}). Here, $J_j^u$ evaluates the worst-case safety margin during the transition, defined as the minimum between the safety function $\ell(x)$ evaluated at the intermediate transition states $x_{\text{sub}}$, and the network prediction at the subsequent discrete state $V_\theta(x^{u_j})$:
\[J_j^u = \min \left( \min_{x_{\text{sub}}} \ell(x_{\text{sub}}), V_\theta(x^{u_j}) \right).\]
While this formulation more strictly enforces safety during intermediate transitions, there are substantial regions of the state space where $\Delta J_j^u$ evaluates exactly to zero, because HJ reachability is inherently a sparse reward problem. For instance, when the robot is already moving safely away from the failure set, both $J_j^{\overline{u}}$ and $J_j^{\underline{u}}$ equal the safety margin at the very first sub-step, so $\Delta J_j^u = 0$. In these regions, the selection rule in Eqn.~\ref{eq:ct_opt_policy} blindly defaults to the lower bound $\underline{u}_j$. This introduces an artificial bias into the learned optimal policies, as the ideal behavior should instead seek to transition to a next state that maximizes the overall safety value.

To account for this, we introduce a value-based tie-breaking mechanism. When the coefficient is estimated as exactly zero, we explicitly evaluate the resulting next-state values to determine the optimal action.
Again, let $x^{\overline{u}_j}$ and $x^{\underline{u}_j}$ denote the next discrete-time states obtained by setting only the $j$-th control dimension to its upper or lower bound, respectively, while holding the remaining action dimensions fixed at their mean values. The updated probing teachers are as follows:
\begin{equation}
\label{eq:ct_opt_policy_with_tie_breaking}
\resizebox{\linewidth}{!}{$
\hat{u}_j^* =
\begin{cases}
\overline{u}_j, & \Delta J_j^u > 0, \\
\underline{u}_j, & \Delta J_j^u < 0, \\
\arg\max_{u_j \in \{\underline{u}_j, \overline{u}_j\}} V_\theta(x^{u_j}), & \Delta J_j^u = 0,
\end{cases}
\qquad
\hat{d}_j^* =
\begin{cases}
\underline{d}_j, & \Delta J_j^d > 0, \\
\overline{d}_j, & \Delta J_j^d < 0, \\
\arg\min_{d_j \in \{\underline{d}_j, \overline{d}_j\}} V_\theta(x^{d_j}), & \Delta J_j^d = 0.
\end{cases}
$}
\end{equation}

\subsection{Triple-Sided Probing}
While bang-bang control is theoretically optimal for continuous-time, control-disturbance-affine systems, it can introduce chattering behavior in discrete-time implementations. For instance, a quadrotor navigating along the centerline of a narrow gate might repeatedly oscillate between maximum and minimum thrust, resulting in a chattery and inefficient control profile.

To mitigate these discretization artifacts, we provide an alternative ``triple-sided'' probing strategy. We introduce a midpoint baseline action, defined as $u^{\text{mid}} = (\overline{u} + \underline{u}) / 2$. For each control dimension $j$, we independently evaluate the value network at the next states resulting from applying the upper bound ($x^{\overline{u}_j}$), the lower bound ($x^{\underline{u}_j}$), and the midpoint ($x^{u^{\text{mid}}_j}$). The optimal probing control $\hat{u}^*_j$ is then selected by maximizing the value function across these three discrete candidate actions:
$$\hat{u}^*_j = \arg\max_{u_j \in \{\underline{u}_j, u^{\text{mid}}_j, \overline{u}_j\}} V_\theta(x^{u_j}).$$

\subsection{Gamma Discounting}
Unlike standard RL-based reachability methods, our framework does not inherently require a discount factor ($\gamma < 1$) to induce a contraction mapping for the Bellman operator. Because we evaluate the problem within a bounded temporal curriculum, the exact, undiscounted optimal value propagates backward smoothly without diverging.

However, incorporating a discount factor $\gamma \in (0, 1]$ can still be advantageous for learning BRAT solutions in practice by providing a denser signal, incentivizing earlier goal reaching. We can seamlessly incorporate this temporal discount factor into our framework, resulting in discounted safety value solutions:
\begin{equation}
\label{eq:bellman_recursion_discounted}
\begin{aligned}
    V_S^*(x, k) &= \max_{\mathbf{u}} \min_{\mathbf{d}} \left[ \min_{\kappa \in [k, K]} \gamma^{\kappa - k} \ell(x_\kappa)  \right], \\
    V_L^*(x, k) &= \min_{\mathbf{u}} \max_{\mathbf{d}} \left[ \min_{\kappa \in [k, K]} \left( \max \left( \gamma^{\kappa - k} g(x_\kappa), \max_{s\in[k,\kappa]} -\gamma^{s - k} \ell(x_s) \right) \right) \right].
\end{aligned}
\end{equation}

\section{Discrete-Time Optimality Gap}
\label{app:optimality_gap}
In general, the optimal control can be interior under DT affine dynamics, which means our method does not converge to the ground-truth reachability solutions even without learning errors. However, the Bellman-Isaacs update with bang-bang teacher control provides an approximation that asymptotically converges to the ground-truth value function as $\Delta t \to 0$. 
For single-player problems, this serves as a strictly conservative approximation. 
Additionally, the bang-bang teacher control is optimal when computing the Backward/Forward Reachable Sets (BRS/FRS) for DT affine systems. 
We will later present the convergence proof, BRS optimality proof, and a worst-case conservatism bound for single-player DT affine systems. 
In practice, cumulative conservatism usually remains much smaller than the worst-case bound because bang-bang policies could approximate time-averaged optimal controls via rapid switching.

\subsection{Empirical Evaluation: Impact of Time Discretization step $\Delta t$}
\label{app:dt_ablation}
We first empirically evaluate the impact of $\Delta t$ on the learned solutions' optimality using two representative robotic systems: a 3D Dubins car liveness problem and a two-link arm safety problem (Fig.~\ref{fig:coverage_table}). 
We use the policy safe set coverage as the metric---the fraction of truly safe states that remain safe under our learned policy rollouts. 
Importantly, when the policy is trained using a time step of $\Delta t$, the trajectory rollout utilizes a control frequency of $\frac{1}{\Delta t}$ while the physical simulation step is performed with a fixed, small time step to match the underlying continuous time dynamics for these two systems. 
Compared to high-precision continuous-time numerical solutions \cite{mitchell2007toolbox}, our learned DT policies maintain near-optimal safe set coverage, except at an aggressive $\Delta t = 0.5$~s. 
\begin{figure}[h]
  \centering
  \includegraphics[width=0.8\linewidth]{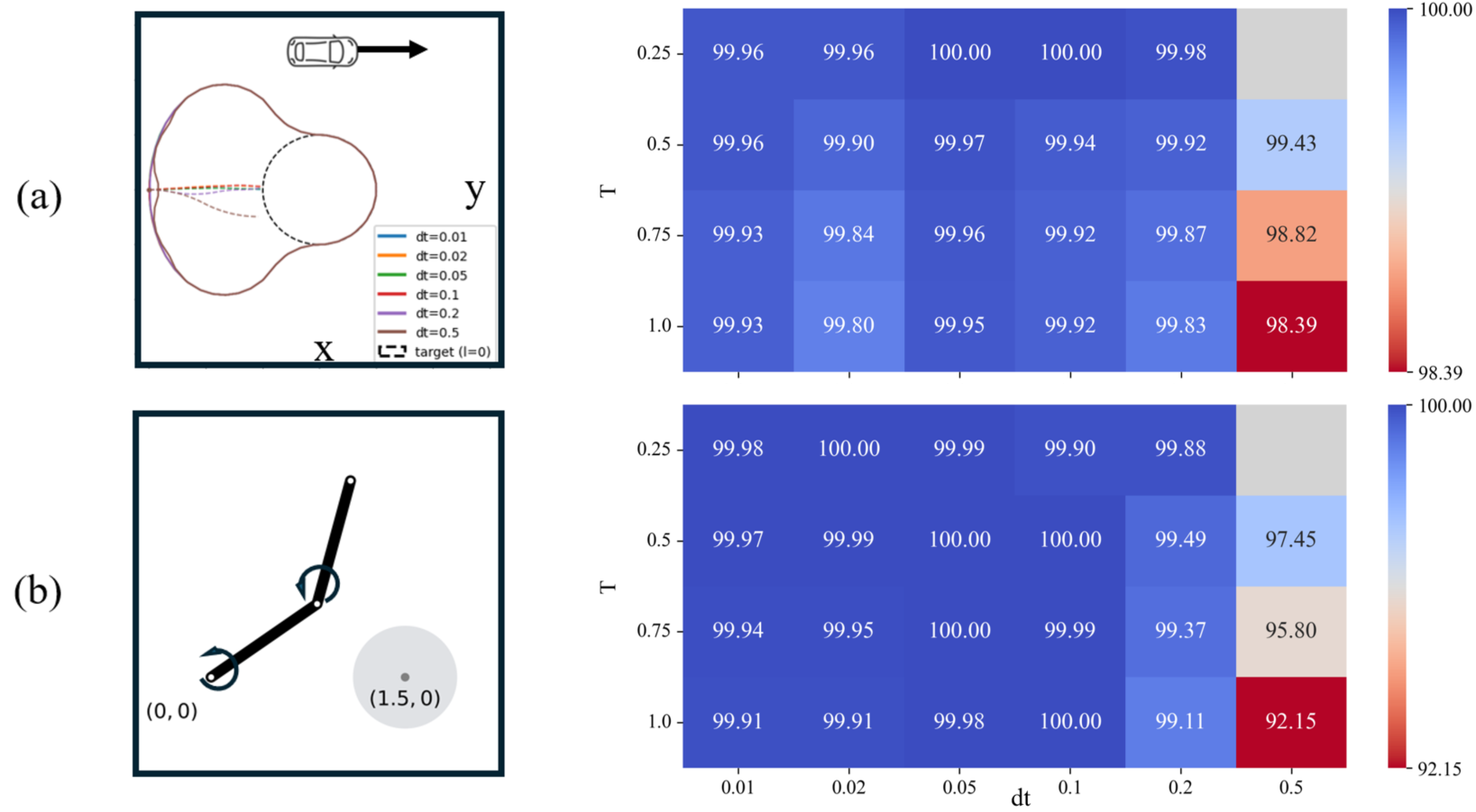}
  \caption{Left: (a) 3D Dubins problem setup with policy safe sets and rollouts ($T=1.0$~s, varying $\Delta t$), and (b) two-link arm problem setup. Right: Safe set coverage percentages across varying $\Delta t$ and $T$. \label{fig:coverage_table}}
  \vspace{-0.75em}
\end{figure}

We further ablate $\Delta t$ on the 10-D Narrow Passage benchmark, where numerical solutions are computationally intractable. The resulting metrics, compiled in Table~\ref{tab:dt_ablation}, track the empirical False Positive Rate (FPR), False Negative Rate (FNR), and total rollout success rate under identical network parameterizations. Similarly, performance is relatively insensitive to the choice of $\Delta t$ across a broad operating range. 
\begin{table}[h]
\centering
\caption{Ablation results isolating the impact of the time discretization step size $\Delta t$ on safety metrics and rollout performance over $20,000$ test states.}
\label{tab:dt_ablation}
\begin{tabular}{lcccc}
\toprule
\textbf{Discretization Step ($\Delta t$)} & \textbf{FPR \%} $\downarrow$ & \textbf{FNR \%} $\downarrow$ & \textbf{Succ \%} $\uparrow$ & $\textbf{Time (h)} \downarrow$\\
\midrule
$\Delta t = 0.025\text{s}$  & $100.0$ & $0.0$ & $82.01$ & $1.9$ \\
$\Delta t = 0.050\text{s}$ & $100.0$ & $0.0$ & $99.71$ & $1.1$\\
$\Delta t = 0.075\text{s}$ & $100.0$ & $0.0$ & $93.49$ & $0.9$\\
$\Delta t = 0.100\text{s}$  & $100.0$ & $0.0$ & $96.31$ & $0.8$\\
$\Delta t = 0.125\text{s}$  & $100.0$ & $0.0$ & $99.99$ & $0.7$\\
$\Delta t = 0.150\text{s}$  & $100.0$ & $0.0$ & $99.98$ & $0.6$\\
$\Delta t = 0.300\text{s}$  & $0.0$ & $100.0$ & $78.75$ & $0.6$\\
\bottomrule
\end{tabular}
\end{table}

As the Narrow Passage has a longer total horizon of 8 seconds, the tradeoff caused by $\Delta t$ can be better observed. An excessively large $\Delta t$ may introduce significant linearization errors, distorting the local approximation of the system physics and leading to degraded teacher action labels near complex boundary manifolds. Conversely, while an infinitesimally small $\Delta t$ theoretically recovers the true analytical directional derivative, it can practically bottleneck optimization by drastically expanding the number of discrete steps required to span the terminal horizon $K$. This temporal expansion increases the computational burden of multi-step student rollouts, increasing total training time, and renders the backward curriculum more vulnerable to bootstrapping error propagation across adjacent temporal windows. Furthermore, as $\Delta t$ diminishes, the probing teacher becomes highly vulnerable to approximation errors in the learned value function. 

Concretely, moderate discretization steps ($0.1$--$0.15$s) provide the best trade-off between computational efficiency and control performance, achieving the highest rollout success rates while substantially reducing training time. Only when the discretization becomes excessively coarse ($\Delta t=0.3$s) does performance degrade noticeably, indicating that the local dynamics approximation becomes insufficient for accurate value propagation.

\subsubsection{Continuum Limit and Exact Recovery via Bang-Bang Control}
We present here the proof that the conservativeness approaches $0$ when $\Delta t \rightarrow 0$. 
\begin{theorem}[Exact Continuum Recovery]
Consider a continuous-time nonlinear affine system $\dot{x} = f(x) + g(x)u$, where $f$ and $g$ are Lipschitz continuous, and the control is constrained to a convex polytope $u \in \mathcal{U}$. Let $\mathcal{V}$ be the set of extreme bang-bang vertices of $\mathcal{U}$. 

Let $x^*(t)$ be the optimal continuous-time trajectory generated by an arbitrary interior control signal $u^*(t) \in \mathcal{U}$ over a fixed horizon $T$. For any discrete time step $\Delta t > 0$, let $x^{BB}_k$ be the state of the discrete-time system generated by a strict bang-bang control sequence $v_k \in \mathcal{V}$. 

As the discretization step $\Delta t \to 0$, there exists a bang-bang sequence such that the discrete-time trajectory uniformly converges to the exact continuous-time optimal trajectory:
\begin{equation}
    \lim_{\Delta t \to 0} \max_{k \in [0, T/\Delta t]} \|x^{BB}_k - x^*(k\Delta t)\| = 0
\end{equation}
\end{theorem}

\begin{proof}
The proof relies on approximating the continuous interior control integral with a high-frequency bang-bang sequence and bounding the trajectory divergence using Grönwall's Inequality.

\textbf{Step 1: Local Control Matching} \\
Partition the time horizon $T$ into intervals of duration $\tau$. Each interval $\tau$ contains $M$ discrete micro-steps of size $\Delta t$ (i.e., $\tau = M \Delta t$).

Over any interval $\tau$, we define the average continuous interior control:
\begin{equation}
    \bar{u}_\tau = \frac{1}{\tau} \int_0^\tau u^*(t) dt \quad \in \mathcal{U}
\end{equation}
By Carathéodory's Theorem, this interior point $\bar{u}_\tau$ can be written exactly as a convex combination of at most $n_u + 1$ bang-bang vertices: $\bar{u}_\tau = \sum_{i=1}^{n_u+1} \lambda_i v_i$, where $\sum \lambda_i = 1$.

We construct the discrete bang-bang sequence $v_k$ by applying each vertex $v_i$ for exactly $\lambda_i M$ steps. By design, the cumulative control effort over the interval $\tau$ perfectly matches the average interior control:
\begin{equation}
    \int_{t_i}^{t_{i+1}} v(s) ds = \int_{t_i}^{t_{i+1}} u^*(s) ds \label{eqn:end-point-equality}
\end{equation}
Furthermore, the maximum instantaneous pointwise divergence between the controls is bounded by the diameter of the control space: $\|v(s) - u^*(s)\| \le C_u$, where $C_u = 2 \max_{v \in \mathcal{V}} \|v\|$.

\textbf{Step 2: Bounding the State Divergence} \\
We evaluate the spatial divergence between the continuous trajectory $x^*(t)$ and the discrete bang-bang trajectory $x^{BB}(t)$ at an arbitrary time $t \in [0, T]$. Subtracting the system dynamics and applying the triangle inequality alongside the Lipschitz constant $L_f$ yields:
\begin{gather}
        \|x^{BB}(t) - x^*(t)\| \le \int_0^t L_f \|x^{BB}(s) - x^*(s)\| ds \nonumber \\
        + \left\| \int_0^t g(x^*(s)) \big(v(s) - u^*(s)\big) ds \right\|
\end{gather}

To evaluate the integration error, let $F_{max} = \max_{x, u} \|f(x) + g(x)u\|$ denote the maximum magnitude of the state derivative. We partition the integral into intervals of size $\tau$. Over a single complete interval $[t_i, t_{i+1}]$, we isolate the variation of $g(x)$ by centering it on the frozen matrix $g(x^*(t_i))$:
\begin{equation}
    E_i = \int_{t_i}^{t_{i+1}} \Big[ g(x^*(t_i)) + \big( g(x^*(s)) - g(x^*(t_i)) \big) \Big] \big(v(s) - u^*(s)\big) ds
\end{equation}

Because the unweighted control integrals match exactly over $\tau$, the frozen component evaluates to zero (\ref{eqn:end-point-equality}). The remaining matrix variation is bounded by the state drift ($F_{max}\tau$), the Lipschitz constant of the control matrix ($L_g$), and the maximum pointwise control difference ($C_u$):
\begin{equation}
    \|E_i\| \le \int_{t_i}^{t_{i+1}} (L_g F_{max} \tau) (C_u) ds = (L_g F_{max} C_u) \tau^2
\end{equation}

For any time $t$, the number of elapsed intervals is $t/\tau$. Summing the bound across these intervals yields the accumulated integration error up to time $t$, which is $(t L_g F_{max} C_u) \tau$. Substituting this directly into the divergence inequality gives:
\begin{equation}
    \|x^{BB}(t) - x^*(t)\| \le \int_0^t L_f \|x^{BB}(s) - x^*(s)\| ds + (t L_g F_{max} C_u) \tau
\end{equation}

Because the bounding term is monotonically increasing with $t$, we apply the integral form of Grönwall's Inequality to resolve an exponential upper bound on the spatial divergence at the final time $T$ (see the remark):
\begin{equation}
    \|x^{BB}(T) - x^*(T)\| \le (T L_g F_{max} C_u \tau) \cdot e^{L_f T}
\end{equation}

\textbf{Step 3: The Limit} \\
To complete the proof, we tie the macroscopic interval $\tau$ to the discrete step size $\Delta t$. Let $\tau = \sqrt{\Delta t}$. Thus, as $\Delta t \to 0$, the interval $\tau \to 0$, and the number of discrete micro-steps per interval $M = \tau/\Delta t = 1/\sqrt{\Delta t} \to \infty$. 

Substituting $\tau = \sqrt{\Delta t}$ into the Grönwall bound, and absorbing the time-invariant system constants ($T, L_g, F_{max}, C_u$) into the asymptotic notation, yields:
\begin{equation}
    \|x^{BB}(T) - x^*(T)\| \le \mathcal{O}(\sqrt{\Delta t}) \cdot e^{L_f T}
\end{equation}
Taking the limit as $\Delta t \to 0$, the spatial divergence strictly converges to zero. 

Therefore, the discrete bang-bang reachable set perfectly converges to the continuous-time interior reachable set, the exact optimal values are recovered without conservatism, and the proof is complete.

\begin{remark}[Application of Grönwall's Inequality]
The integral form of Grönwall's Inequality states that for a continuous non-negative function $u(t)$, if it satisfies the integral inequality 
\begin{equation*}
    u(t) \le \alpha(t) + \int_0^t \beta u(s) ds
\end{equation*}
where $\beta \ge 0$ is a constant and $\alpha(t)$ is a non-decreasing continuous function, then the function is bounded by $u(t) \le \alpha(t) e^{\beta t}$. 

In our formulation, the variables map exactly to this standard form:
\begin{itemize}
    \item The bounding target is the spatial divergence: $u(t) = \|x^{BB}(t) - x^*(t)\|$.
    \item The constant multiplier is the dynamics Lipschitz bound: $\beta = L_f$.
    \item The forcing term is the accumulated integration error: $\alpha(t) = (t L_g F_{max} C_u) \tau$.
\end{itemize}
Crucially, because all constants within $\alpha(t)$ are strictly positive, the accumulated error grows linearly with time $t$. This guarantees that the forcing term $\alpha(t)$ is monotonically non-decreasing, which satisfies the prerequisite condition for the clean exponential bound. Substituting our variables yields the continuous bound $\|x^{BB}(t) - x^*(t)\| \le (t L_g F_{max} C_u \tau) e^{L_f t}$, which we evaluate at the terminal horizon $t = T$.
\end{remark}
\end{proof}

\subsection{Teacher Optimality for Safety BRS}
We now show that a bang-bang control sequence exists as an optimal solution for computing the safety BRS for discrete-time linear systems.

\begin{theorem}[Optimality of Bang-Bang Control for Linear Safety BRS]
Consider a discrete-time linear affine system $x_{k+1} = A x_k + B u_k + c$, subject to control constraints $u_k \in \mathcal{U}$, where $\mathcal{U} \subset \mathbb{R}^{m}$ is a compact convex polytope. Let $\mathcal{V}$ denote the finite set of extreme vertices of $\mathcal{U}$. 

Assume the failure function $\ell(x)$ is convex, and consider that the $K$-step safety BRS is characterized by the terminal-state failure function. There exists an optimal control sequence $\mathbf{u}^* = (u_0^*, u_1^*, \dots, u_{K-1}^*)$ that resides entirely within the bang-bang vertex set $\mathcal{V}^K$.
\end{theorem}

\begin{proof}
\textbf{Step 1: The Affine State Mapping} \\
For a linear affine system, the terminal state $x_K$ after $K$ steps can be expressed explicitly as an affine function of the initial state $x_0$ and the stacked control sequence $\mathbf{U} = (u_0, \dots, u_{K-1})$:
\begin{equation}
    x_K(\mathbf{U}) = A^K x_0 + \sum_{i=0}^{K-1} A^{K-1-i} c + \sum_{i=0}^{K-1} A^{K-1-i} B u_i
\end{equation}
Let $\mathcal{U}^K = \mathcal{U} \times \dots \times \mathcal{U}$ be the admissible space for the full control sequence. Because $\mathcal{U}$ is a compact convex polytope, the Cartesian product $\mathcal{U}^K$ is also a compact convex polytope. The extreme vertices of $\mathcal{U}^K$ are precisely the sequences composed exclusively of the bang-bang vertices: $\text{ext}(\mathcal{U}^K) = \mathcal{V}^K$.

\textbf{Step 2: Convexity of the Objective Function} \\
In safety BRS computation, the objective is to maximize the terminal safety function, defined by $V_K(x_0) = \max_{\mathbf{U} \in \mathcal{U}^K} \ell(x_K(\mathbf{U}))$. We define the composition function $\Phi(\mathbf{U}) = \ell(x_K(\mathbf{U}))$. 

Because $x_K(\mathbf{U})$ is an affine mapping with respect to $\mathbf{U}$, and $\ell(x)$ is assumed to be convex, their composition $\Phi(\mathbf{U})$ is also convex over $\mathcal{U}^K$.

\textbf{Step 3: Extreme Point Attainment} \\
The problem now reduces to maximizing the convex function $\Phi(\mathbf{U})$ over the compact convex polytope $\mathcal{U}^K$. By Bauer's Maximum Principle, the maximum of a convex function over a compact convex set is always attained at one (or more) of the set's extreme points. 

As established in Step 1, the extreme points of the sequence space $\mathcal{U}^K$ are precisely the sequences composed entirely of bang-bang vertices from $\mathcal{V}$. Therefore, there always exists an optimal control sequence $\mathbf{u}^*$ that resides entirely within the bang-bang vertex set $\mathcal{V}^K$, completing the proof.
\end{proof}

\begin{remark}[Scope of Bang-Bang Optimality]
This existence result fundamentally relies on the maximization of a convex objective over a polyhedral control sequence set. In linear DT \textit{liveness} formulations (target reaching), preventing target overshoot often necessitates interior control values rather than exclusively using boundary vertices. Furthermore, this optimality does not generally hold for nonlinear affine systems and BRT computation.
\end{remark}

\subsection{Discrete-Time Bang-Bang Conservatism Bound for Single-Player BRT}

We now present the worst-case conservatism bound for bang-bang control policies in single-player discrete-time systems, followed by its formal proof.

\begin{theorem}[Discrete-Time Bang-Bang Conservatism Bound for Single-Player BRT]
Consider a discrete-time system $x_{k+1} = f(x_k, u_k)$, where the dynamics $f$ are Lipschitz continuous with respect to the control input with constant $L_f$, and the control is constrained to a compact convex set $\mathcal{U}$. Let $\mathcal{V} \subset \mathcal{U}$ be the discrete set of extreme bang-bang vertices. Assume the optimal value function $V^*(x)$ is Lipschitz continuous with respect to the state with constant $L_V$.

For any discrete horizon $K$, let $V_0^{BB}(x)$ be the value function achieved by strictly applying bang-bang control sequences, and let $V_0^*(x)$ be the optimal value function achieved by interior controls under the single-player Backward Reachable Tube (BRT) formulation. The worst-case conservatism gap is strictly bounded by:
$$V_0^*(x) - V_0^{BB}(x) \le L_f L_V C_u K$$
where $C_u = \max_{u \in \mathcal{U}} \min_{v \in \mathcal{V}} ||u - v||$ is the maximum control quantization error.
\end{theorem}

\begin{remark}
This bound represents an extreme worst-case scenario and is highly conservative compared to the empirical performance observed in most practical problems, as demonstrated earlier. An intuitive example of this extreme conservatism is a system traversing a narrow, knife-edge mountain ridge; any strict bang-bang control application would immediately result in falling off the (unrecoverable) cliff, yielding a large conservatism gap. In practice, however, bang-bang policies can often closely approximate time-averaged optimal interior controls via rapid switching, resulting in significantly tighter actual performance gaps.
\end{remark}

\begin{proof}
At an arbitrary time step $k$, let $u^* \in \mathcal{U}$ be the optimal interior control that satisfies the Bellman optimality equation for the control player maximizing safety:
$$V_k^*(x) = \min(\ell(x), V_{k+1}^*(f(x, u^*)))$$
We select a discrete bang-bang control $v \in \mathcal{V}$ that minimizes the spatial distance to $u^*$, guaranteeing $||v - u^*|| \le C_u$.
Because the bang-bang policy maximizes the value over all available vertices, the bang-bang value function at step $k$ is inherently lower-bounded by selecting this specific vertex $v$:
$$V_k^{BB}(x) \ge \min(\ell(x), V_{k+1}^{BB}(f(x, v)))$$

To evaluate the conservatism gap at step $k$, we subtract both sides from the ground-truth value $V_k^*(x) = \min(\ell(x), V_{k+1}^*(f(x, u^*)))$, yielding:
$$V_k^*(x) - V_k^{BB}(x) \le \min(\ell(x), V_{k+1}^*(f(x, u^*))) - \min(\ell(x), V_{k+1}^{BB}(f(x, v)))$$

Because the min operator $\min(\ell(x), V)$ is strictly non-expansive, it caps the maximum difference between the two terms without amplifying the divergence. Thus, the difference between the minimums is upper-bounded by the difference of their internal variables:
$$V_k^*(x) - V_k^{BB}(x) \le V_{k+1}^*(f(x, u^*)) - V_{k+1}^{BB}(f(x, v))$$

By adding and subtracting an intermediate term $V_{k+1}^*(f(x, v))$, we separate the carried-over error gap from the single-step deviation penalty:
$$V_k^*(x) - V_k^{BB}(x) \le [V_{k+1}^*(f(x, v)) - V_{k+1}^{BB}(f(x, v))] + [V_{k+1}^*(f(x, u^*)) - V_{k+1}^*(f(x, v))]$$

Applying the Lipschitz bounds for both the value function ($L_V$) and the system dynamics ($L_f$) to the single-step penalty provides the recursive error bound:
$$V_k^*(x) - V_k^{BB}(x) \le [V_{k+1}^*(f(x, v)) - V_{k+1}^{BB}(f(x, v))] + L_V L_f C_u$$

At the terminal time step $K$, both value functions evaluate the exact same terminal cost boundary condition $\ell(x)$, making the initial conservatism gap zero. Iterating this recursive inequality backward from the terminal step $K$ to the initial step $0$ accumulates the constant single-step penalty exactly $K$ times. Therefore, the worst-case value function conservatism gap scales linearly with the discrete horizon:
$$V_0^*(x) - V_0^{BB}(x) \le L_f L_V C_u K$$
which completes the proof.
\end{proof}


\section{Ablations}
\subsection{Teacher-Mixing Fractions}
\label{app:teacher_mixing}
To evaluate the role of teacher-student mixing, we ablate the teacher injection fraction $\alpha$ within each temporal window, comparing student-only ($\alpha=0$), teacher-only ($\alpha=1.0$), balanced mixture ($\alpha=0.5$), and decaying ($1.0 \rightarrow 0.0$) schedules.

As shown in Table~\ref{tab:teacher_ablation}, the constant mixture achieves the smallest success-rate deviation from the numerical solution ($0.206\%$) while remaining competitive on both rollout cost metrics. Teacher-only supervision produces the lowest rollout-cost discrepancy as expected, but exhibits the largest success-rate gap, suggesting that excessive reliance on teacher targets may bias policy learning. Student-only and decaying schedules achieve intermediate performance across all metrics. Overall, the results suggest that performance is relatively robust to the choice of teacher schedule, although the constant mixture consistently provides the closest match to the numerical baseline in terms of success rate.


\begin{table}[h!]
\centering
\caption{Ablation of the teacher mixing fraction $\alpha$ on the Dubins3D benchmark.
Metrics report the absolute difference between our learned policy and the numerical solution over $20{,}000$ test states. $\Delta$Succ denotes the success-rate gap, while $\Delta$Cost Mean and $\Delta$Cost Std denote the differences in mean rollout cost and rollout-cost standard deviation, respectively. Lower values indicate closer agreement with the numerical baseline.}
\label{tab:teacher_ablation}
\resizebox{\textwidth}{!}{
\begin{tabular}{lccc}
\toprule
\textbf{Curriculum Schedule ($\alpha$)} & \textbf{$\Delta$ Succ \%} $\downarrow$ & \textbf{$\Delta$ Cost Mean} $\downarrow$ & \textbf{$\Delta$ Cost Std} $\downarrow$ \\
\midrule
Pure Student ($\alpha = 0.0$) & $0.2210\% \pm 0.0859\%$ & $0.0356 \pm 0.0015$ & $0.0772 \pm 0.0021$ \\
Pure Teacher ($\alpha = 1.0$) & $0.3362\% \pm 0.0608\%$ & $\mathbf{0.0274 \pm 0.0013}$ & $\mathbf{0.0643 \pm 0.0025}$ \\
Decaying Curriculum ($\alpha: 1.0 \rightarrow 0.0$) & $0.2267\% \pm 0.0607\%$ & $0.0299 \pm 0.0070$ & $0.0699 \pm 0.0061$ \\
Constant Mixture ($\alpha = 0.5$) & $\mathbf{0.2056\% \pm 0.0856\%}$ & $0.0312 \pm 0.0018$ & $0.0706 \pm 0.0026$ \\
\bottomrule
\end{tabular}
}
\end{table}

\subsection{Component-Wise Ablations}
Compared to our 99.71\% success rate on Narrow Passage, removing components led to significant performance degradation---in order of impact: replacing the windowed architecture with a monolithic network (9.32\%), omitting window fine-tuning (65.57\%), and removing asymmetric FP weighting (85.53\%). This highlights that windowing and finetuning are particularly critical components for long-horizon accuracy of reachable sets.


\section{Additional Benchmarking Results}

The following benchmarks are included to further evaluate the framework on adversarial zero-sum games beyond the primary experiments presented in the main text.

\subsection{Evaluation on Two-Player Zero-Sum Games}
\label{app:zero_sum}

\subsubsection{4D Human-Robot Collision Avoidance}
We first formulate a 4D human-robot collision avoidance scenario as a Worst-Case Backward Reachable Tube (BRT). In this game, an ego vehicle (modeled as a 4D Dubins car) attempts to avoid an adversarial human agent (modeled as a 3D Dubins vehicle). The system is modeled in a relative coordinate frame with the following dynamics:
\begin{equation*}
\begin{aligned}
\dot{x}_{\text{rel}} &= -v_{\text{ego}} + v_{\text{human}}\cos(\theta_{\text{rel}}) + \omega_{\text{ego}} y_{\text{rel}} \\
\dot{y}_{\text{rel}} &= v_{\text{human}}\sin(\theta_{\text{rel}}) - \omega_{\text{ego}} x_{\text{rel}} \\
\dot{\theta}_{\text{rel}} &= \omega_{\text{human}} - \omega_{\text{ego}} \\
\dot{v}_{\text{ego}} &= a_{\text{ego}}
\end{aligned}
\end{equation*}
The 4D state vector is defined as $x = [x_{\text{rel}}, y_{\text{rel}}, \theta_{\text{rel}}, v_{\text{ego}}]^\top$, where $x_{\text{rel}}$ and $y_{\text{rel}}$ denote the relative position of the human with respect to the ego vehicle, $\theta_{\text{rel}}$ is the relative heading, and $v_{\text{ego}}$ is the linear velocity of the ego vehicle. The state bounds are given as $[-3.0, 3.0]^2\times [-\pi, \pi] \times [0.1, 1.0]$. We use Euler integration to obtain the discrete-time dynamics.

The ego vehicle acts as the control player aiming to maximize safety, with control inputs $u = [\omega_{\text{ego}}, a_{\text{ego}}]^\top$ bounded by $\omega_{\text{ego}} \in [-1.0, 1.0]$ and $a_{\text{ego}} \in [-2.0, 2.0]$. The human acts as the disturbance player aiming to cause a collision, with a higher turn rate $d=\omega_{\text{human}} \in [-2.0, 2.0]$. The human is assumed to move at a constant linear velocity $v_{\text{human}} = 1.0$.
\begin{table}[htbp]
    \centering
    \caption{Quantitative safety evaluation across 20,000 random rollouts in the 4D collision avoidance environment.}
    \begin{tabular}{lcc}
        \toprule
        \textbf{Method} & \textbf{False Positives (FP)} $\downarrow$ & \textbf{False Negatives (FN)} $\downarrow$ \\
        \midrule
        OptimizedDP & 1569 & \textbf{10} \\
        \textbf{Ours} & \textbf{206} & 35 \\
        \bottomrule
    \end{tabular}
    \label{tab:quant_4d_game}
\end{table}

The safety set is defined by avoiding a collision ball of radius $r_{\text{goal}} = 0.5$ around the ego vehicle: $\ell(x) = \sqrt{x_{\text{rel}}^2 + y_{\text{rel}}^2} - r_{\text{goal}}$.

In this scenario, both players possess distinct advantages---the ego vehicle has greater longitudinal flexibility through variable speed, while the human has a higher maximum turn rate. This asymmetric control authority introduces sharp value gradients that cause numerical artifacts in the grid-based OptimizedDP solver \cite{bui2022optimizeddp}, as illustrated by the qualitative value and cost function heatmaps in Fig.~\ref{fig:rel4d_odp}. Quantitatively, the value-policy consistency is further evaluated with the numbers of FP and FN among 20,000 trajectory rollouts, as shown in Table~\ref{tab:quant_4d_game}. Our method yields a significantly smaller number of FP while being slightly more conservative with a higher number of FN.

\begin{figure}[htbp]
    \centering
    \begin{subfigure}{0.48\textwidth}
        \centering
        \includegraphics[width=\textwidth]{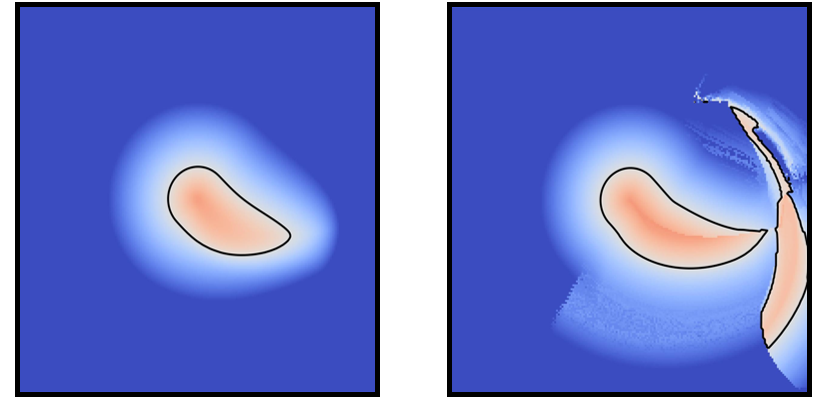}
        \caption{OptimizedDP solution}
        \label{fig:rel4d_odp}
    \end{subfigure}\hfill
    \begin{subfigure}{0.48\textwidth}
        \centering
        \includegraphics[width=\textwidth]{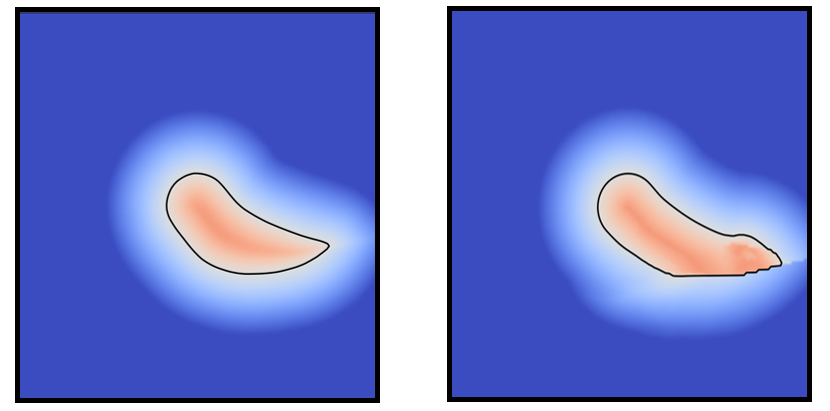}
        \caption{Our solution}
        \label{fig:rel4d_ours}
    \end{subfigure}\hfill
    \caption{Qualitative comparison of the 4D human-robot collision avoidance game. For both methods, we display the predicted value function (left sub-panel) alongside the reachability rollout cost evaluated from its induced policy (right sub-panel). The heatmap slices are taken on the x-y plane where $\theta_{\text{rel}}=\pi/2$ and $v_{\text{ego}} =1.0$. The grid-based OptimizedDP solution exhibits severe numerical artifacts due to the asymmetric control authority between players. In contrast, our neural framework maintains a relatively more consistent safety boundary that closely matches the empirical rollout cost.}
    \label{fig:qualitative_4d_game}
\end{figure}

\subsubsection{10D Multi-Agent Pursuit-Evasion}
We scale our evaluation to a high-dimensional reach-avoid problem featuring one evader (modeled as a 4D Dubins car) and two cooperative pursuers (modeled as 3D Dubins cars). The game takes place in a bounded $10 \times 10$ field containing a central circular obstacle with a radius of $1$.

The 10D state vector concatenates the states of all three agents: $x = [x_e, y_e, \theta_e, v_e, x_{p_1}, y_{p_1}, \theta_{p_1}, x_{p_2}, y_{p_2}, \theta_{p_2}]^\top$. The spatial coordinates for all agents are bounded by the field size $x, y \in [-5.0, 5.0]$, headings by $\theta \in [-\pi, \pi]$, and the evader's velocity by $v_e \in [0.0, 2.0]$.
The two pursuers act as the joint control player (minimizer) aiming to capture the evader, with control inputs $u = [\omega_{p_1}, \omega_{p_2}]^\top$ bounded by $\omega_{p_i} \in [-1.0, 1.0]$. The evader acts as the disturbance player aiming to escape, with inputs $d = [\omega_e, a_e]^\top$ bounded by $\omega_e \in [-1.2, 1.2]$ and $a_e \in [-2.0, 2.0]$. Both pursuers move at a constant linear velocity $v_{\text{const}} = 1.0$.
Again, the problem is formulated with asymmetric control authority, leading to intriguing optimal strategies and numerical stiffness during training.

The continuous-time dynamics of the multi-agent system are given by
\begin{equation*}
\begin{aligned}
\dot{x}_e &= v_e\cos(\theta_e), \quad &\dot{y}_e &= v_e\sin(\theta_e), \quad &\dot{\theta}_e &= \omega_e, \quad &\dot{v}_e &= a_e \\
\dot{x}_{p_i} &= v_{\text{const}}\cos(\theta_{p_i}), \quad &\dot{y}_{p_i} &= v_{\text{const}}\sin(\theta_{p_i}), \quad &\dot{\theta}_{p_i} &= \omega_{p_i}, \quad &i &\in \{1, 2\}
\end{aligned}
\end{equation*}
and we apply Euler integration to obtain the discrete-time dynamics.

The target set is successfully reached if the evader is caught (comes within a capture radius of $r_{\text{catch}} = 0.75$ of any pursuer) or if the evader crashes into an obstacle/wall:
\[g(x) = \min \left( \min_{i \in \{1, 2\}} \left( \|p_e - p_{p_i}\| - r_{\text{catch}} \right), \text{SDF}(p_e) \right),\]
where $\text{SDF}(\cdot)$ denotes the signed distance function to the boundary and the central circular obstacle. The failure set is triggered if the pursuers fail the mission by colliding with the obstacles or each other:
\[\ell(x) = \min \left( \min_{i \in \{1, 2\}} \text{SDF}(p_{p_i}), \|p_{p_1} - p_{p_2}\| - 2r_{\text{robot}} \right).\]

\begin{table}[htbp]
    \centering
    \caption{Quantitative catching rates for the 10D pursuit-evasion game. We evaluate the performance of our method against the MPPI baseline by pairing different combinations of pursuer and evader policies.}
    \begin{tabular}{llcc}
        \toprule
        & & \multicolumn{2}{c}{\textbf{Evader Policy}} \\
        \cmidrule(lr){3-4}
        & & \textbf{MPPI} & \textbf{Ours} \\
        \midrule
        \multirow{2}{*}{\textbf{Pursuer Policy}} & \textbf{MPPI} & 36\% & 0\% \\
        & \textbf{Ours} & \textbf{88\%} & \textbf{42\%} \\
        \bottomrule
    \end{tabular}
    \label{tab:quant_10d_game}
\end{table}

To evaluate the optimality of our learned policies, we benchmark them against a Model Predictive Path Integral (MPPI) baseline \cite{7487277} in a cross-play evaluation. The MPPI pursuers operate independently, each individually minimizing its cumulative distance to the evader while incorporating an SDF penalty term for safety awareness. Conversely, the MPPI evader operates by simply maximizing its distance from the pursuers while avoiding obstacles.

Table~\ref{tab:quant_10d_game} presents the quantitative catching rates across all combinations of pursuer and evader policies. Our framework demonstrates dominant performance in both the evading and cooperative pursuing roles. When our learned evader is deployed against the baseline MPPI pursuers, it successfully escapes in 100\% of the 50 trials. Conversely, our joint pursuer policy captures the MPPI evader in 88\% of the rollouts, vastly outperforming the baseline MPPI pursuers, which only achieve a 36\% success rate against their own evader. Finally, when putting our learned pursuers against our learned evader, the catch rate settles at 42\%.

\subsection{Qualitative Analysis}
A primary motivation for adopting our discrete-time formulation is the flexibility to employ standard, highly stable neural network architectures (e.g., ReLU MLPs) to represent the value function. In contrast, PINN-based solvers rely heavily on sinusoidal MLPs to accurately represent the value function gradients required by the HJI PDE.
While the sinusoidal networks excel at handling differentiable boundary conditions, they often struggle to approximate complex non-smooth safety boundaries. 

As an example, the safety boundary in the Quadrotor Gate Avoidance problem is highly complex and non-smooth, defined by the SDF to the gate coupled with the velocity bounds. 
To demonstrate this advantage, we provide qualitative heatmap slices of the learned value functions in Fig.~\ref{fig:heatmaps}. The PINN-based baseline (DeepReachMPC), constrained by its sinusoidal network, exhibits severe high-frequency oscillatory artifacts (ripples) and fails entirely to resolve the true value function near the gate. Conversely, our gradient-free framework stably propagates the safety boundary backward over the temporal curriculum.

\section{Environment Setup and Training Configurations}
\label{app:experimental_setup}

\subsection{Environment Definitions}

\subsubsection{Narrow Passage}

The Narrow Passage environment models two vehicles navigating in opposite directions along a shared road segment that contains a stationary obstacle.

\textbf{State and Action Space.}
The joint state is $x = [x_1, y_1, \theta_1, v_1, \phi_1,\; x_2, y_2, \theta_2, v_2, \phi_2] \in \mathbb{R}^{10}$, where $(x_i, y_i)$ is the position, $\theta_i \in [-\pi, \pi]$ is the heading, $v_i \in [0.1, 7.0]$\,m/s is the speed, and $\phi_i \in [-0.3\pi, 0.3\pi]$\,rad is the steering angle of vehicle $i \in \{1,2\}$.  The control input is $u = [a_1, \dot\phi_1, a_2, \dot\phi_2]$ with $|a_i| \leq 2.0$\,m/s$^2$ and $|\dot\phi_i| \leq 3\pi$\,rad/s.

\textbf{Dynamics.}
Each vehicle follows the kinematic bicycle model with wheelbase $L = 2.0$\,m:
\begin{align}
    \dot x_i &= v_i \cos\theta_i, &
    \dot y_i &= v_i \sin\theta_i, &
    \dot \theta_i &= \frac{v_i \tan\phi_i}{L}, &
    \dot v_i &= a_i, &
    \dot\phi_i &= \dot\phi_i^{\text{ctrl}}.
\end{align}

\textbf{Target and Obstacle Sets.}
Vehicle~1 must reach goal $g_1 = (6.0, -1.2)$\,m and vehicle~2 must reach $g_2 = (-6.0, 1.2)$\,m, both within distance $L$:
\begin{equation}
    \ell_{\text{target}}(x) = \max\!\left(\|[x_1,y_1] - g_1\| - L,\;\|[x_2,y_2] - g_2\| - L\right).
\end{equation}
The failure set encodes four constraint types, all of which must be satisfied to be safe:
\begin{itemize}
    \item \textbf{Road curbs:} $y_i \in [\text{curb}_\text{lo} + L/2,\; \text{curb}_\text{hi} - L/2]$ with $\text{curb}_\text{lo} = -2.8$\,m, $\text{curb}_\text{hi} = 2.8$\,m.
    \item \textbf{Lateral bounds:} $x_i \in [-10.0 + L/2,\; 10.0 - L/2]$\,m.
    \item \textbf{Stranded vehicle:} $\|[x_i, y_i] - p_{\text{obs}}\| \geq L$ with $p_{\text{obs}} = (0.0, -1.8)$\,m.
    \item \textbf{Mutual collision:} $\|[x_1,y_1] - [x_2,y_2]\| \geq L$.
\end{itemize}
The avoid function aggregates all constraints as $\ell_{\text{avoid}}(x) = -w \cdot \min_k d_k(x)$, where $\{d_k\}$ are the signed distances to each constraint boundary and $w = 10.0$.

\textbf{Neural Network Inputs.}
The 8 non-heading state components are normalized to $[-1, 1]$; the two headings $\theta_1, \theta_2$ are encoded as $(\sin\theta_i, \cos\theta_i)$, yielding a 12-dimensional input vector $\phi(x) \in \mathbb{R}^{12}$.

\subsubsection{Quadrotor Gate Avoidance}

This environment is formulated as a safety (BRT) problem: a quadrotor approaching a gate along the $+x$ direction must avoid crashing into the gate frame.

\textbf{State and Action Space.}
The state is $x = [p_x, p_y, p_z, q_w, q_x, q_y, q_z, v_x, v_y, v_z, \omega_x, \omega_y, \omega_z] \in \mathbb{R}^{13}$, where $(p_x,p_y,p_z)$ is position, $(q_w,q_x,q_y,q_z)$ is the quaternion, $(v_x,v_y,v_z)$ is linear velocity, and $(\omega_x,\omega_y,\omega_z)$ is angular velocity.  State bounds are $p_x \in [-4.0, 1.0]$\,m, $p_y, p_z \in [-1.5, 1.5]$\,m, and $\omega_i \in [-5.0, 5.0]$\,rad/s.  The 4D control input is $u = [\Delta T, \dot\omega_x, \dot\omega_y, \dot\omega_z]$, where $\Delta T \in [-9.8, 9.8]$\,N is a collective thrust perturbation around the hover thrust $T_h = 9.8$\,N, and $|\dot\omega_{x,y}| \leq 8.0$\,rad/s$^2$, $|\dot\omega_z| \leq 4.0$\,rad/s$^2$.

\textbf{Dynamics.}
The continuous-time equations of motion are:
\begin{align}
    \dot p &= v, \\
    \dot q &= \tfrac{1}{2}\,\Omega(q)\,\omega, \\
    \dot v_x &= 2(q_w q_y + q_x q_z)\,\frac{k T_c}{m}, \quad
    \dot v_y = 2(-q_w q_x + q_y q_z)\,\frac{k T_c}{m}, \quad
    \dot v_z = -g + (1 - 2q_x^2 - 2q_y^2)\,\frac{k T_c}{m}, \\
    \dot\omega_x &= \dot\omega_x^{\text{ctrl}} - \tfrac{5}{9}\,\omega_y\omega_z, \quad
    \dot\omega_y = \dot\omega_y^{\text{ctrl}} + \tfrac{5}{9}\,\omega_x\omega_z, \quad
    \dot\omega_z = \dot\omega_z^{\text{ctrl}},
\end{align}
where $\Omega(q)$ is the quaternion kinematic matrix, $T_c = \Delta T + T_h$, $m = 1.0$\,kg, $k = C_T / m = 1.0$, and $g = 9.8$\,m/s$^2$.

\textbf{Obstacle Set.}
The gate is an infinite wall at $x=0$ with a rectangular opening $|p_y| \leq 0.5$\,m, $|p_z| \leq 0.5$\,m and a width $0.1$\,m.  The quadrotor body is a ball of radius $r = 0.15$\,m.  The unsafe set combines gate collision with velocity bound violations: $v_x \in [0.5, 4.5]$\,m/s, $v_y, v_z \in [-2.0, 2.0]$\,m/s, and $\omega_i \in [-4.5, 4.5]$\,rad/s.  A tanh-scaled signed distance is used to sharpen the gate boundary.

\textbf{Neural Network Inputs.}
All 13 state dimensions are used; position and velocity components are linearly normalized to $[-1,1]$ while the quaternion is kept as a raw unit vector, yielding $\phi(x) \in \mathbb{R}^{13}$.

\subsubsection{Publisher-Subscriber System ($N$D)}

This benchmark scales our method to 40D and 80D by instantiating a weakly nonlinear system with $N \in \{40, 80\}$ states.

\paragraph{State and Action Space.}
$x \in [-1, 1]^N$, control $u \in [-u_{\max}, u_{\max}]^{N-1}$ with $u_{\max} = 0.5$.  There is no adversarial disturbance.

\paragraph{Dynamics.}
\begin{equation}
    \dot x = A x + B u + \eta(x),
\end{equation}
where $A = -0.5\,I_N - e_0 \mathbf{1}_{1:N-1}^\top$, $B = [0;\; 0.4\,I_{N-1}]$, and the nonlinear term is
\begin{equation}
    \eta_0(x) = \mu\sin(\alpha x_0)\,x_0^2, \qquad
    \eta_{i>0}(x) = -\gamma\,x_0^2\,x_i,
\end{equation}
with $\gamma = 20.0$ and $\mu = \alpha = 0$. More details are available in \cite{sharpless2024linear}.

\paragraph{Target Set.}
A scaled ellipsoid in $\mathbb{R}^N$:
\begin{equation}
    \ell(x) = \tfrac{1}{2}\!\left(\|\mathrm{diag}(\varepsilon)\,x\|^2 - R^2\right) \leq 0,
    \quad \varepsilon = [\sqrt{N-1},\,1,\ldots,1]^\top,\quad R = \sqrt{N-1}\cdot 0.25.
\end{equation}

\paragraph{Neural Network Inputs.}
The state is already in $[-1,1]^N$ and is passed directly: $\phi(x) = x \in \mathbb{R}^N$.

\subsubsection{Dubins Car (3D)}
The Dubins car is a planar reach-avoid problem: a unicycle system with a fixed speed must reach a target region while avoiding circular obstacles.

\textbf{State and Action Space.}
The state is $x = [p_x, p_y, \theta] \in [-5.0, 5.0]^2 \times [-\pi, \pi]$, where $(p_x, p_y)$ is the planar position and $\theta$ is the heading. The scalar control is the turning rate $\omega \in [-\omega_{\max}, \omega_{\max}]$ with $\omega_{\max} = 1.0$\,rad/s. Forward speed is fixed at $v = 1.0$\,m/s.

\textbf{Dynamics.}
\begin{equation}
    \dot p_x = v\cos\theta, \qquad \dot p_y = v\sin\theta, \qquad \dot\theta = \omega.
\end{equation}

\textbf{Target and Obstacle Sets.}
The target is a disk of radius $r_{\text{target}} = 0.5$\,m centred at the
origin:
\begin{equation}
    \ell_{\text{target}}(x) = \|[p_x, p_y]\| - r_{\text{target}}.
\end{equation}

The environment contains $N_{\text{obs}} = 5$ circular obstacles with centres $c_i \in \mathbb{R}^2$ and radii $r_i$, given in Table~\ref{tab:dubins_obstacles}. The signed avoid function takes the value of the most dangerous obstacle at each state:
\begin{equation}
    \ell_{\text{avoid}}(x) =
        \max_{i=1}^{N_{\text{obs}}}
        \bigl(r_i - \|[p_x, p_y] - c_i\|\bigr).
\end{equation}
A state is unsafe if $\ell_{\text{avoid}}(x) > 0$, i.e.\ the vehicle is inside at least one obstacle.

\begin{table}[h!]
\centering
\caption{Default obstacle parameters for the Dubins car environment.}
\label{tab:dubins_obstacles}
\begin{tabular}{cccc}
\toprule
$i$ & $c_i^x$ & $c_i^y$ & $r_i$ \\
\midrule
1 & $-0.70$ & $\phantom{-}0.20$ & $0.30$ \\
2 & $\phantom{-}1.20$ & $-1.50$ & $0.35$ \\
3 & $\phantom{-}1.80$ & $\phantom{-}1.00$ & $0.50$ \\
4 & $-2.00$ & $\phantom{-}1.50$ & $0.40$ \\
5 & $-1.50$ & $-2.00$ & $0.25$ \\
\bottomrule
\end{tabular}
\end{table}

\textbf{Neural Network Inputs.}
Position is linearly normalized to $[-1,1]^2$ and the heading is encoded as $(\sin\theta, \cos\theta)$, yielding $\phi(x) \in \mathbb{R}^4$.

\subsection{Network Architectures}
\label{app:architectures}

All experiments employ a window-partitioned architecture consisting of independent value and policy networks for each temporal window. The value network estimates the reachability value function, while the policy network jointly parameterizes both control and disturbance actions through a shared latent representation. For a horizon partitioned into $M$ windows, the framework instantiates $M$ value networks and $M$ policy networks, each responsible for a local temporal segment.

\textbf{Windowed Value Network.}
Each temporal window is represented by a time-conditioned residual MLP that approximates the local value function $V(x,k)$. The architecture consists of:

\begin{enumerate}
    \item \textbf{Input stem:}
    \begin{equation*}
        [k,\phi(x)]
        \rightarrow
        \mathrm{Linear}(1+\phi_{\mathrm{dim}},W_V)^{L_V}
        \rightarrow
        \mathrm{ReLU},
    \end{equation*}
    where $\phi(x)$ denotes the state encoding.
    \item \textbf{Residual trunk:}
    $L$ residual blocks, each consisting of a two-layer MLP with a skip connection.
    \item \textbf{Output head:}
    \begin{equation*}
        \mathrm{Linear}(W_V,1),
    \end{equation*}
    producing the scalar value estimate $V_\theta(x,k)$.
\end{enumerate}

Unless otherwise specified, all value networks use width $W_V=512$, depth $L_V=8$, and are optimized using AdamW with learning rate $3\times10^{-4}$.

\textbf{Windowed Vector-Quantized Policy Network.}
Each temporal window is associated with a policy network that jointly predicts control and disturbance actions. The network maps states into a latent action embedding which is quantized onto the vertices of the bang-bang action set.

\begin{enumerate}
    \item \textbf{Encoder:}
    \begin{equation*}
    [k,\phi(x)]
    \rightarrow
    (\mathrm{Linear}+\mathrm{ReLU})^{L_\pi}
    \rightarrow
    \mathrm{Linear}(W_\pi,d_u+d_d),
    \end{equation*}
    producing a continuous embedding
    \begin{equation*}
    z_e \in \mathbb{R}^{d_u+d_d}.
    \end{equation*}
    \item \textbf{Vector quantization:}
    The embedding is projected onto a fixed (implicit) codebook containing all vertices of the admissible bang-bang action set. The nearest codebook entry
    \begin{equation*}
    z_q=\arg\min_{c\in\mathcal{C}}
    \|z_e-c\|_2
    \end{equation*}
    is selected and returned through a straight-through estimator (STE), allowing gradients to propagate through the continuous embedding during training.
    \item \textbf{Action decoding:}
    The quantized embedding is interpreted as the joint control-disturbance action vector and scaled to the corresponding control and disturbance limits during rollout execution.
\end{enumerate}

Unless otherwise specified, all policy networks use width $W_\pi=512$, depth $L_\pi=5$, and are optimized using AdamW with learning rate $3\times10^{-4}$.

\begin{table}[h!]
\centering
\caption{Default architecture and training hyperparameters used across all experiments.}
\label{tab:default_hparams}
\begin{tabular}{lc}
\toprule
\textbf{Hyperparameter} & \textbf{Value} \\
\midrule
\multicolumn{2}{c}{\textit{Architecture}} \\
\midrule
Value network width $W$ & 512 \\
Value network depth $L$ & 8 \\
Policy network width $W_\pi$ & 512 \\
Policy network depth $L_\pi$ & 5 \\
\midrule
\multicolumn{2}{c}{\textit{Optimization}} \\
\midrule
Optimizer & AdamW \\
Learning rate & $3\times10^{-4}$ \\
Batch size & 4096 \\
\midrule
\multicolumn{2}{c}{\textit{Value Learning}} \\
\midrule
TD($\lambda$) coefficient & 0.5 \\
Maximum rollout length $M$ & 5 \\
Discount factor $\gamma$ & 1.0 \\
Teacher mixing fraction $\alpha$ & 0.5 \\
Boundary-correction LR & $5\times10^{-5}$ \\
Anchor loss weight $\lambda_{\mathrm{anchor}}$ & 5.0 \\
\bottomrule
\end{tabular}
\end{table}

\subsection{Training Configuration}

Unless otherwise specified, all experiments use the following training configuration.

\textbf{Optimization.}
Value and policy networks are optimized using AdamW with a learning rate of $3\times10^{-4}$ and batch size 4096. Policy and value updates are alternated throughout training, with each phase terminated according to a Bellman-error convergence criterion.

\textbf{Rollout Targets.}
Value targets are generated using a TD($\lambda$) mixture with $\lambda=0.5$ and a maximum rollout horizon of $M=5$ steps. All experiments use $\gamma=1.0$.

\textbf{State Sampling.}
Training batches combine multiple sampling strategies to improve boundary resolution. Unless otherwise specified, 30\% of samples are drawn uniformly from the state space, 15\% near target boundaries, 15\% near obstacle boundaries, and 40\% near the current reachability boundary. Boundary samples are generated using a boundary-aware replay buffer with band half-width 0.1 and capacity 4096.

\textbf{Temporal Curriculum.}
Training proceeds backward from the terminal boundary condition using the windowed temporal curriculum described in Sec.~\ref{sec:method}. Each temporal window is initialized from the subsequent converged window and optimized independently. Before freezing a window, a boundary-correction phase is performed using rollout-derived targets generated by the learned student policies. Boundary correction uses a reduced learning rate of $5\times10^{-5}$ and an anchor loss weight of $\lambda_{\mathrm{anchor}}=5.0$.

\textbf{Policy Learning.}
Policies are trained through supervised imitation of teacher actions obtained via gradient-free bang-bang value probing. During rollout target generation, teacher and student actions are mixed with a fixed teacher injection fraction of $\alpha=0.5$.

\begin{table}[h!]
\centering
\caption{Benchmark-specific configurations. $d_x$: state dimension; $d_\phi$: NN input dimension excluding the time input (after encoding); $d_u$/$d_d$: control/disturbance dimensions.}
\label{tab:benchmark_configs}
\begin{tabular}{lccccccc}
\toprule
\textbf{Benchmark} & \textbf{Horizon (s)} & \textbf{Windows} & $\mathbf{\Delta t}$ & $\mathbf{d_x}$ & $\mathbf{d_\phi}$ & $\mathbf{d_u}$ & $\mathbf{d_d}$ \\
\midrule
Dubins Car        & $4.0$ & $4$ & $0.02$ & $3$ & $4$ & $1$ & $0$ \\
Narrow Passage    & $8.0$ & $16$ & $0.05$ & $10$ & $12$ & $4$ & $0$ \\
Quadrotor         & $1.0$ & $8$ & $0.01$ & $13$ & $13$ & $4$ & $0$ \\
Pursuit-Evasion   & $8.0$ & $8$ & $0.05$ & $10$ & $13$ & $2$ & $2$ \\
40D Network       & $1.0$ & $2$ & $0.01$ & $40$ & $40$ & $39$ & $0$ \\
80D Network       & $1.0$ & $2$ & $0.01$ & $80$ & $80$ & $79$ & $0$ \\
\bottomrule
\end{tabular}
\end{table}

\section{Details on Hardware Setup and Experiments}
\label{app:hardware}

This section provides additional details for the race car BRT training and safety filtering.
\subsection{Vehicle Dynamics and Problem Setup}
We adopt the F1Tenth hybrid dynamics from the F1Tenth simulator \cite{pmlr-v123-o-kelly20a}. To train a robust safety value function capable of handling varying levels of real-world uncertainty, we augment the state space with a disturbance scale parameter, $d_s \in [0, 0.3]$.

The state is then given by $x=(p_x, p_y, \phi_s , v , \theta_{yaw} , \omega_{yaw} , \beta_{slip}, d_s)$ where $\phi_s \in [-0.4189, 0.4189]$, $v \in [0.5, 10]$, $\omega_{yaw} \in  [-5, 5]$, and $\beta_{slip} \in [-0.8, 0.8]$. Here, $\phi_s$ denotes the steering angle, $v$ is the longitudinal velocity, $\theta_{yaw}$ is the yaw angle, $\omega_{yaw}$ is the yawing rate, and $\beta_{slip}$ denotes the slip angle. The control inputs are the rate of steering angle and the longitudinal acceleration, denoted as $u=[\dot{\phi}, a] \in [-3.2,3.2] \times [-9.51, 9.51]$. The disturbances include lateral drift scaler $v_l \in [-1,1]$ and control authority deduction $d_\phi$ and $d_a$, whose bounds are scaled by an extra conditional dimension $d_s$. The longitudinal acceleration is further capped when the current velocity $v$ is high, and the system switches to kinematic mode when the velocity is too low. 
We increase the switching speed for the dynamic mode to $1.5$~m/s to match the speed limit of the hardware RC car. Namely, the F1Tenth car is in its kinematic mode if the speed is less than 1.5~m/s. Otherwise, the dynamic mode will be used. We now present the continuous-time dynamics for each mode. Euler integration is employed to provide discrete-time dynamics.

The dynamic-mode dynamics are:
\[
\mathbf{f} = \begin{bmatrix}
v \cos(\theta_{yaw} + \beta_{slip}) - v_l d_s (|v|+1) \sin(\theta_{yaw}) \\
v \sin(\theta_{yaw} + \beta_{slip}) + v_l d_s (|v|+1) \cos(\theta_{yaw})\\
\dot{\phi}- d_\phi d_s \\
a - d_a d_s\\
\omega_{yaw} \\
-\frac{\mu m}{v I (l_r + l_f)} \left( l_f^2 C_{Sf} (g l_r - a h) + l_r^2 C_{Sr} (g l_f + a h) \right) \omega_{yaw} \\
+ \frac{\mu m}{I (l_r + l_f)} \left( l_r C_{Sr} (g l_f + a h) - l_f C_{Sf} (g l_r - a h) \right) \beta_{slip} \\
+ \frac{\mu m}{I (l_r + l_f)} l_f C_{Sf} (g l_r - a h) \phi_s, \\
\left( \frac{\mu}{v^2 (l_r + l_f)} \left( C_{Sr} (g l_f + a h) l_r - C_{Sf} (g l_r - a h) l_f \right) - 1 \right) \omega_{yaw} \\
- \frac{\mu}{v (l_r + l_f)} \left( C_{Sr} (g l_f + a h) + C_{Sf} (g l_r - a h) \right) \beta_{slip} \\
+ \frac{\mu}{v (l_r + l_f)} C_{Sf} (g l_r - a h) \phi_s
\end{bmatrix}.
\]
where the vehicle parameters are borrowed directly from the F1Tenth simulator:
\begin{itemize}
    \item \( \mu \): Surface friction coefficient: \(1.0489\)
    \item \( C_{Sf} \): Cornering stiffness coefficient, front: \( 4.718 /\, \text{rad} \)
    \item \( C_{Sr} \): Cornering stiffness coefficient, rear: \( 5.4562 /\, \text{rad} \)
    \item \( l_f \): Distance from center of gravity to front axle: \( 0.15875 \, \text{m} \)
    \item \( l_r \): Distance from center of gravity to rear axle: \( 0.17145 \, \text{m} \)
    \item \( h \): Height of center of gravity: \( 0.074 \, \text{m} \)
    \item \( m \): Total mass of the vehicle: \( 3.74 \, \text{kg} \)
    \item \( I \): Moment of inertia of the entire vehicle about the \( z \)-axis: \( 0.04712 \, \text{kg} \cdot \text{m}^2 \)
\end{itemize}

The kinematic-mode dynamics are:
\[
\mathbf{f} = \begin{bmatrix}
v \cos(\theta_{yaw}) - v_l d_s (|v|+1) \sin(\theta_{yaw})\\
v \sin(\theta_{yaw}) + v_l d_s (|v|+1) \cos(\theta_{yaw})\\
\dot{\phi} - d_\phi d_s\\
a - d_a d_s\\
\frac{v}{l_r + l_f} \tan(\phi_s) \\
\frac{a}{l_r + l_f} \tan(\phi_s) + \frac{v}{(l_r + l_f) \cos^2(\phi_s)} \dot{\phi} \\
0
\end{bmatrix}.
\]

\subsection{Network Architecture and Inputs}
To effectively process high-dimensional observations, our neural network architectures decouple visual feature extraction from the other input features. The network inputs consist of a $128 \times 128$ ego-centric BEV occupancy map, alongside a 5-dimensional proprioceptive state vector $(\phi, v, \omega_{yaw}, \beta_{slip}, d_s)$ and the time $t$, forming a $16,390$-dimensional input space.

To process the visual input, we employ a CNN encoder, which consists of five convolutional layers, each utilizing a $3 \times 3$ kernel, a stride of 2, a padding of 1, and ReLU activations. This architecture progressively downsamples the $1 \times 128 \times 128$ spatial occupancy map into a $128 \times 4 \times 4$ feature tensor. The tensor is subsequently flattened and passed through a linear projection layer to yield a dense, $256$-dimensional embedding. 

The $256$-dimensional BEV embedding is then concatenated with the proprioceptive states and the time variable. This joint feature vector is passed into a ResNet with 8 residual blocks, each maintaining the 512-neuron width with ReLU activations, to predict the value function or optimal actions.
Importantly, the entire architecture is trained end-to-end from scratch, including the BEV encoder.

\subsection{Discrete-Time CBF Filter}
To ensure safety during online deployment, we guard the task-driven MPPI nominal control, $u_{\text{nom}}$, using a sampling-based variant of the DCBF filter \cite{agrawal2017discrete}. 

First, we define the calibrated safety value function as:
\begin{equation}
    V_{\text{cal}}(x) = V_\theta(x,0) - \delta
\end{equation}
where $V_\theta(x,0)$ is the raw output of the learned value network and $\delta$ is the calibration threshold obtained following \cite{pmlr-v242-lin24a}. To maintain safety, the DCBF strictly bounds the decaying rate of the safety value based on the current value. We define this required minimum safety value for the next step as:
\begin{equation}
    V_{\text{req}}(x) = \max \left( (1 - \gamma_{\text{DCBF}} \Delta t) V_{\text{cal}}(x), 0 \right)
\end{equation}
where $\gamma_{\text{DCBF}}=1.0$ is the relaxation rate. 

Since our method does not approximate the value function gradients, we utilize a sampling-based strategy to evaluate the DCBF constraint over a batch of candidate controls, $\mathcal{U}_{\text{batch}}$. The subset of safe candidate controls is formulated as:
\begin{equation}
    \mathcal{U}_{\text{safe}} = \left\{ u \in \mathcal{U}_{\text{batch}} \mid V_{\text{cal}}\left(x_{k+1}^{u,d_\phi}\right) \ge V_{\text{req}}(x) \right\}
\end{equation}

The filtered control output is then determined by a switching logic. If the safe control set is non-empty, the filter selects the safe candidate closest to the nominal control. If no sampled control satisfies the safety constraint, the filter triggers the robust fallback recovery policy $\pi^u_\phi$:
\begin{equation}
    u_{\text{filtered}} = 
    \begin{cases} 
        \arg\min_{u \in \mathcal{U}_{\text{safe}}} \| u - u_{\text{nom}} \|_2, & \text{if } \mathcal{U}_{\text{safe}} \neq \emptyset \\
        \pi^u_\phi(x_{\text{clamped}},0), & \text{otherwise}
    \end{cases}
\end{equation}
where $x_{\text{clamped}}$ restricts the state within the strict bounds of the training domain prior to evaluating the fallback policy. Theoretically, if the BRT solution is exact and all encountered disturbances remain strictly within the training bounds, the safe control set is guaranteed to be non-empty. In practice, however, neural network approximation errors or out-of-distribution observations and disturbances may occasionally render the safe set empty. The fallback recovery logic is explicitly introduced to maintain persistent system safety despite these practical imperfections.

\end{appendices}